\def\year{2022}\relax
\documentclass[letterpaper]{article} %
\usepackage{aaai22}  %
\usepackage{times}  %
\usepackage{helvet}  %
\usepackage{courier}  %
\usepackage[hyphens]{url}  %
\usepackage{graphicx} %
\usepackage{natbib}  %
\usepackage{caption} %
\DeclareCaptionStyle{ruled}{labelfont=normalfont,labelsep=colon,strut=off} %
\usepackage{algorithm}
\usepackage{algorithmic}
\usepackage{microtype}      %
\usepackage{subcaption}
\usepackage{xcolor}
\usepackage{bm}
\usepackage{amsmath}
\usepackage{mathtools}
\usepackage{amsfonts}
\usepackage{ifthen}
\usepackage{amsthm}
\newtheorem{theorem}{Theorem}

\usepackage{booktabs}       %
\usepackage{amsfonts}       %
\usepackage{nicefrac}       %
\usepackage{tikz}
\tikzstyle{arrow} = [thick,->,>=stealth,line width=0.5pt]

\usepackage{adjustbox}

\usepackage{amsmath,amsfonts,bm}

\def\eqref#1{equation~\ref{#1}}

\def\1{\bm{1}}

\DeclareMathAlphabet{\mathsfit}{\encodingdefault}{\sfdefault}{m}{sl}
\SetMathAlphabet{\mathsfit}{bold}{\encodingdefault}{\sfdefault}{bx}{n}

\DeclareMathOperator*{\argmax}{arg\,max}

\usepackage{textcomp}
\usepackage{siunitx}
\usepackage{wrapfig}
\usepackage{newfloat}
\usepackage{listings}
\floatstyle{ruled}
\newfloat{listing}{tb}{lst}{}
\floatname{listing}{Listing}
\title{Training a Resilient Q-Network against Observational Interference }
\author{Chao-Han Huck Yang$^{1}$, I-Te Danny Hung$^{2}$, Yi Ouyang$^{3}$, Pin-Yu Chen$^{4}$}
\affiliations{Georgia Institute of Technology$^{1}$, Columbia University$^{2}$, Preferred Networks America$^{3}$, IBM Research AI$^{4}$\\
\small \texttt{huckiyang@gatech.edu, ih2320@columbia.edu, ouyangyi@gmail.com, pin-yu.chen@ibm.com}}

\usepackage{bibentry}

\begin{document}

\maketitle

\begin{abstract}
Deep reinforcement learning (DRL) has demonstrated impressive performance in various gaming simulators and real-world applications.
In practice, however, a DRL agent may receive faulty observation by abrupt interferences such as black-out, frozen-screen, and adversarial perturbation. How to design a resilient DRL algorithm against these rare but mission-critical and safety-crucial scenarios is an essential yet challenging task. In this paper, we consider a deep q-network (DQN)  framework training with an auxiliary task of observational interferences such as artificial noises.
Inspired by causal inference for \textbf{observational interference}, we 
propose a causal inference based DQN algorithm called causal inference Q-network (CIQ).
We evaluate the performance of CIQ in several benchmark DQN environments with different types of interferences as auxiliary labels.
Our experimental results show that the proposed CIQ method could achieve higher performance and more resilience against observational interferences. 
\end{abstract}

\section{Introduction}

Deep reinforcement learning (DRL) methods have shown enhanced performance, gained widespread applications~\citep{mnih2015human,mnih2016asynchronous,  silver2017mastering}, and improved robot learning~\citep{gu2017deep} in navigation systems~\citep{tai2017virtual,  nagabandi2018learning}.
However, most successful demonstrations of these DRL methods are usually trained and deployed under well-controlled situations. In contrast, real-world use cases often encounter inevitable observational uncertainty~\citep{grigorescu2020survey, hafner2018learning, moreno2018neural} from an external attacker~\citep{ huang2017adversarial} or noisy sensor~\citep{fortunato2017noisy,  lee2018bayesian}. For examples, playing online video games may experience sudden black-outs or frame-skippings due to network instabilities, and driving on the road may encounter temporary blindness when facing the sun.
Such \textbf{an abrupt interference on the observation could cause serious issues} for DRL algorithms.
Unlike other machine learning tasks that involve only a single mission at a time (e.g., image classification), an RL agent has to deal with a dynamic~\citep{schmidhuber1992learning} or even learn from latent states with generative models~\citep{schmidhuber1991reinforcement,jaderberg2017reinforcement, ha2018world, hafner2018learning, lynch2020learning} to anticipate future rewards in complex environments. 
Therefore, DRL-based systems are likely to propagate and even enlarge risks (e.g., delay and noisy pulsed-signals on sensor-fusion~\citep{yurtsever2020survey, johansen2015estimation}) induced from the uncertain interference.

In this paper, we investigate the \emph{resilience} ability of an RL agent to withstand unforeseen, rare, adversarial and potentially catastrophic interferences, and to recover and adapt by improving itself in reaction to these events. 
We consider a resilient generative RL framework with observational interferences as an auxiliary task.
At each time, the agent's observation is subjected to a type of sudden interference at a predefined possibility. 
Whether or not an observation has interfered is referred to as the interference label.

Specifically, to train a resilient agent, we provide the agent with the interference labels during training. For instance, the labels could be derived from some uncertain noise generators recording whether the agent observes an intervened state at the moment as a binary causation label. By applying the labels as an \emph{intervention} into the environment, the RL agent is asked to learn a binary causation label and embed a latent state into its model. However, when the trained agent is deployed in the field (i.e., the testing phase), the agent only receives the interfered observations but is agnostic to interference labels and needs to act resiliently against the interference.

For an RL agent to be resilient against interference, the agent needs to diagnose observations to make the correct inference about the reward information.
To achieve this, the RL agent has to reason about what leads to desired rewards despite the irrelevant intermittent interference. To equip an RL agent with this reasoning capability, we exploit the causal inference framework. Intuitively, a causal inference model for observation interference uses an unobserved confounder \citep{pearl2009causality, pearl2019seven, pearl1995testability, saunders2018trial, bareinboim2015bandits, zhang2020causala, khemakhem2021causal} to capture the effect of the interference on the rewards (outcomes) collected from the environment. In recent works, RL is also showing additional benefits incorporating generative causal modeling, such as providing interpretability~\citep{madumal2020explainable}, treatment estimation~\citep{zhang2020designing, zhang2021bounding}, imitation learning~\citep{zhang2020causalb}, enhanced invariant prediction~\citep{zhang2020invariant}, and generative model for transfer learning~\citep{killian2020counterfactually}.

When such a confounder is available, the RL agent can focus on the confounder for relevant reward information and make the best decision.
As illustrated in Figure \ref{fig:figure1}, we propose a causal inference based DRL algorithm termed causal inference Q-network (CIQ).
During training, when the interference labels are available, the CIQ agent will implicitly learn a causal inference model by embedding the confounder into a latent state.
At the same time, the CIQ agent will also train a Q-network on the latent state for decision making. Then at testing, the CIQ agent will make use of the learned model to estimate the confounding latent state and the interference label. The design of CIQ is inspired by causal inference on state variable and using treatment switching method~\citep{shalit2017estimating} to learn latent variable by incorporating observational interference.

The history of latent states is combined into a causal inference state, which captures the relevant information for the Q-network to collect rewards in the environment despite of the observational interference. 

\begin{figure}[ht!]
	\centering
	    \includegraphics[width=0.47\textwidth]{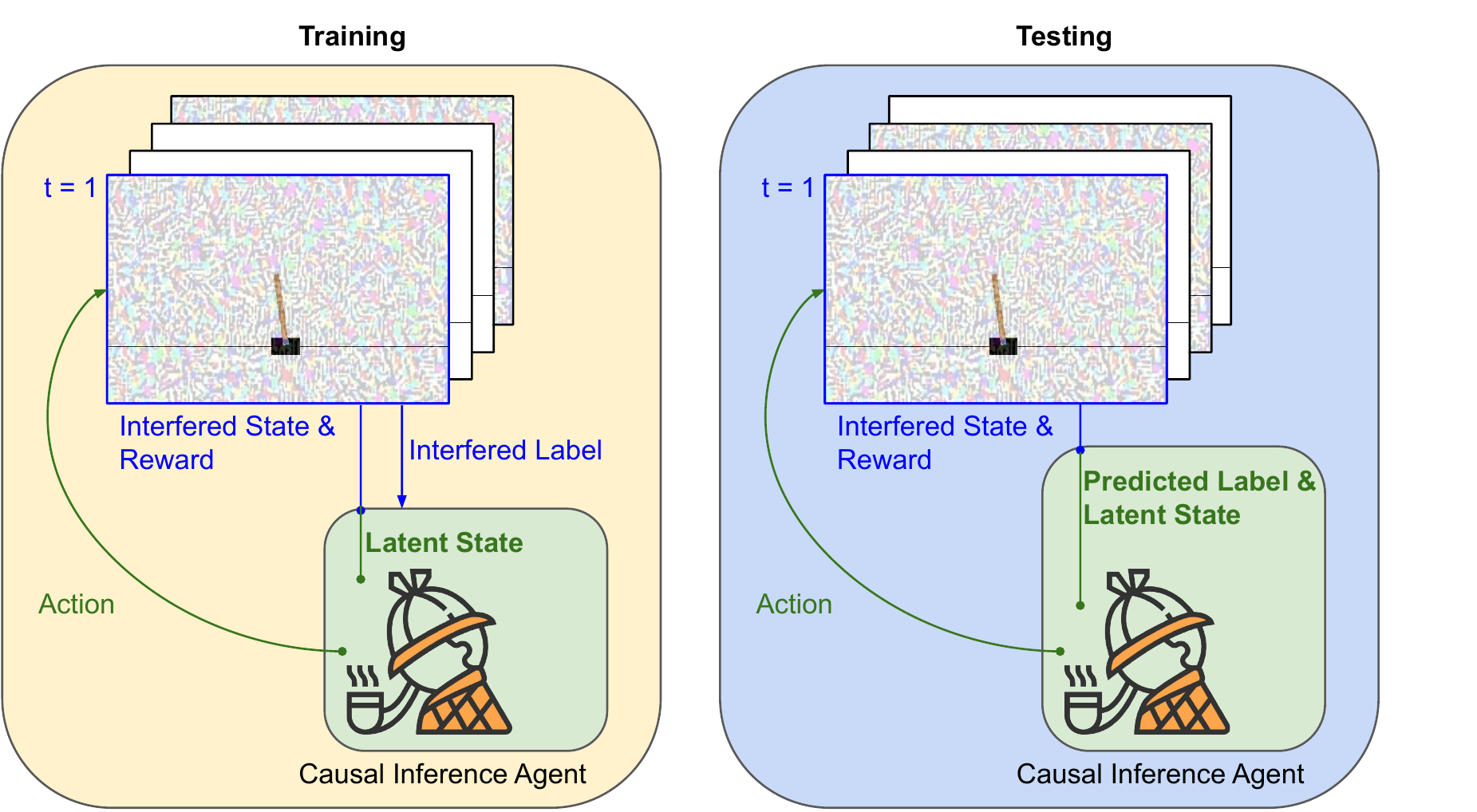}
\vspace{-0.2cm}
   \caption{The proposed causal inference Q-network (CIQ) training and test framework, where the latent state is an unobserved (hidden) confounder variable. We refer the readers to Figure~\ref{fig:diagram} for detailed descriptions on its graphical model.  } 
\label{fig:figure1}
\vspace{-0.2cm}
\end{figure}

In this paper, we evaluate the performance of our method in four environments: 1) Cartpole-v0 -- the continuous control environment~\citep{brockman2016openai}; 2) the 3D graphical Banana Collector~\citep{juliani2018unity}); 3) an Atari environment LunarLander-v2~\citep{brockman2016openai}, and 4)  pixel Cartpole -- visual learning from the pixel inputs of Cartpole.
For each of the environments, we consider four types of interference: (a) black-out, (b) Gaussian noise, (c) frozen screen, and (d) additive noise from adversarial perturbation. 

In the testing phase mimicking the practical scenario that the agent may have interfered observations but is unaware of the true interference labels (i.e., happens or not),
the results show that our CIQ method can perform better and more resilience against all the four types of interference.
Furthermore, to benchmark the level of resilience of different RL models, we propose a new robustness measure, called CLEVER-Q, to evaluate the robustness of Q-network based RL algorithms. The idea is to compute a lower bound on the observation noise level such that the greedy action from the Q-network will remain the same against any noise below the lower bound. According to this robustness analysis, our CIQ algorithm indeed achieves higher CLEVER-Q scores compared with the baseline methods. 

The main contributions of this paper include 1) a framework to evaluate the resilience of DQN-based DRL methods under abrupt observational interferences; 2) the proposed CIQ architecture and algorithm towards training a resilient DQN agent, and 3) 
an extreme-value theory based robustness metric (CLEVER-Q) for quantifying the resilience of Q-network based RL algorithms. 

\section{Related Works}
\textbf{Causal Inference for Generative Reinforcement Learning:} 
Causal inference~\citep{greenland1999causal, pearl2009causality, pearl2016causal, pearl2019seven, robins1995analysis} has been used to empower the learning process under noisy observation and have better interpretability on deep learning models~\citep{shalit2017estimating, louizos2017causal}, also with efforts~\citep{ jaber2019causal, forney2017counterfactual, bareinboim2015bandits,bennett2021off, jung2021estimating} on causal online learning and bandit methods. Defining causation and applying causal inference framework to DRL still remains relatively unexplored. Recent works~\citep{lu2018deconfounding, tennenholtz2019off} study this problem by defining action as one kind of intervention and 
estimating the causal effects. In contrast, we introduce 
observational interference into generative DRL by applying extra noisy and uncertain inventions. Inspired by the treatment switching and representation learning models~\citep{shalit2017estimating, louizos2017causal, helwegen2020improving}, we leverage the causal effect of observational interferences on states, and design an end-to-end structure for learning a \emph{causal-observational} representation evaluating treatment effects on rewards.

\textbf{Adversarial Perturbation:}
An intensifying challenge against deep neural network based systems is adversarial perturbation for making incorrect decisions. Many gradient-based noise-generating methods \citep{goodfellow2014explaining, huang2017adversarial, everett2021neural}  have been conducted for misclassification and mislead an agent's output action. As an example of using DRL model playing Atari games, an adversarial attacker~\citep{lin2017tactics, yang2020enhanced} could jam in a timely and barely detectable noise to maximize the prediction loss of a Q-network and cause massively degraded performance.

\textbf{Partially Observable Markov Decision Processes (POMDPs):} 
Our resilient RL framework can be viewed as a POMDP with interfered observations.
Belief-state methods are available for simple POMDP problems (e.g., plan graph and the tiger problem~\citep{kaelbling1998planning}), but no provably efficient algorithm is available for general POMDP settings~\citep{papadimitriou1987complexity,gregor2018temporal}. Recently, Igl \emph{et. al}~\cite{igl2018deep} have proposed a DRL approach for POMDPs by combining variational autoencoder and policy-based learning, but this kind of methods do not consider the interference labels available during training in our resilient RL framework.

\section{Resilient Reinforcement Learning}
\label{sec_attack}
In this section, we formally introduce our resilient RL framework and provide an extreme-value theory based metric called CLEVER-Q for measuring the robustness of DQN-based methods. 

We consider a sequential decision-making problem where an agent interacts with an environment. 
At each time $t$, the agent gets an observation $x_t$, e.g. a frame in a video environment.
As in many RL domains (e.g., Atari games), we view $s_t = (x_{t-M+1}, \ldots, x_t)$ to be the state of the environment where $M$ is a fixed number for the history of observations.
Given a stochastic policy $\pi$, the agent chooses an action $a_t \sim \pi(s_t)$ from a discrete action space based on the observed state and receives a reward $r_t$ from the environment. 
For a policy $\pi$, define the Q-function
$
    Q^{\pi}(s, a) = \mathbb E \big[ \sum_{t=0}^\infty \gamma^t r_t|s_0 = s,\, a_0 = a, \pi\big]
$
where $\gamma \in (0, 1)$ is the discount factor.
The agent's goal is to find the optimal policy $\pi^*$ that achieves the optimal Q-function given by $Q^{*}(s, a) = \max_\pi Q^{\pi}(s, a)$.

\subsection{Resilience base on an Interventional Perspective} 
\label{sub_41}
To evaluate the resilience ability of RL agents, we introduce additional interference as auxiliary information (as illustrated in Fig \ref{fig:figure1} ) as an empirical process ~\citep{pearl2009causality, louizos2017causal} for observation. Given a type of interference $\mathcal I$, the agent's observation becomes:
\begin{equation}
    x'_t =  F^{\mathcal I}(x_t, i_t)=i_t \times \mathcal I(x_t) + (1 - i_t) \times x_t
    \label{eq:noise:41}
\end{equation}
where $i_t \in \{0, 1\}$ is the label indicating whether the observation is interfered at time $t$ or not, and $\mathcal I(x_t)$ is the interfered observation.

We assume that interference labels $i_t$ follow an i.i.d. Bernoulli process with a fixed interference probability $p^{\mathcal I}$ as a noise level.\footnote{The i.i.d. assumption could be extended to a Markovian dynamic interference model. We show experiments with dynamic interference in Appendix E\ref{sup:sec:markov}.}
For example, when $p^{\mathcal I}$ equals to 10\%, each observational state has a 10\% chance to be intervened under a perturbation. In this work, we consider the original observations, as illustrated in Figure \ref{fig:n3} (a), under four types of interference as described below.

\textbf{Gaussian Noise.}
Gaussian noise or white noise is a common interference to sensory data~\citep{osband2019behaviour, yurtsever2020survey}. The interfered observation becomes $\mathcal I(x_t) = x_t + n_t$ with a zero-mean Gaussian noise $n_t$. The noise variance is set to be the variance of all recorded states as illustrated in Figure \ref{fig:n3} (b). 

\textbf{Adversarial Observation.} 
Following the standard adversarial RL attack setting, we use fast gradient sign method (FGSM) \citep{szegedy2013intriguing} to generate adversarial patterns against the DQN loss~\citep{ huang2017adversarial} as illustrated in Figure \ref{fig:n3} (c). 
The observation is given by 
$\mathcal I(x_t) = x_t + \epsilon \operatorname{sign}\left(\nabla_{x_t} Q(x_t,y;\theta)\right)$
where $y$ is the optimal action by weighting over possible actions.

\textbf{Observation Black-Out.} Off-the-shelf hardware can affect the entire sensor networks as a sensing background~\citep{yurtsever2020survey} over-shoot with $\mathcal I(x_t) = 0$~\citep{yan2016can}. This perturbation is realistic owing to overheat hardware and losing the observational information of sensors.
\label{s3:blackout}

\textbf{Frozen Frame.}
Lagging and frozen frame(s)~\citep{kalashnikov2018qt} often come from limited data communication bottleneck bandwidth. A frozen frame is given by $\mathcal I(x_t) = x_{t-1}$. If the perturbation is constantly present, the frame will remain the first frozen frame since the perturbation happened.  

\begin{figure}[ht!]
\begin{center}
   \includegraphics[width=0.85\linewidth]{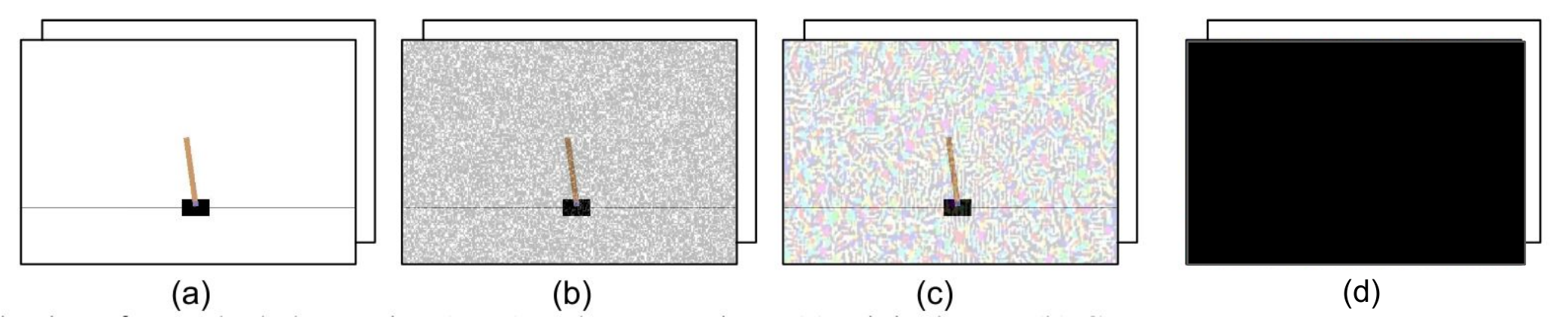}
\end{center}
\vspace{-0.4cm}
\caption{ Visualization of perturbed observation (state) under uncertainty: (a) original state; (b) Gaussian perturbation; (c) adversarial perturbation \cite{huang2017adversarial}, and (d) black-out perturbation (a white-out ablation in the Appendix E). }
\label{fig:n3}
\vspace{-0.5cm}
\end{figure}

\subsection{Resilient Reinforcement Learning Framework}

With observational interference, instead of the actual state $s_t$, the agent only observes $s'_t = (x'_{t-M+1}, \ldots,  x'_t)$.
The agent now needs to choose its actions $a_t \sim \pi(s'_t)$ based on the interfered observation. The resilient RL objective for the agent is to find a policy $\pi$ to maximize rewards in this environment under observational interference. 
Under the resilient framework, the goal of a Q-learning based agent is to learn the relation between $s'_t$ and $Q_t$ where $Q_t(a) =  \max_\pi \mathbb E \big[ \sum_{\tau= t}^\infty \gamma^{(\tau-t)} r_\tau|s'_t,\, a_t = a, \pi\big]$ denotes the Q-values given the interfered observation $s'_t$ at time $t$.

From the RL model and the observation model of Eq. (\ref{eq:noise:41}), the relation among the observation $s'_t$, Q-values $Q_t$, and interference $i_t$ can be described by a causal graphical model (CGM) in Figure \ref{fig:diagram}. 
In the CGM, $z_t=(s_t, i_{t-M+1}, \ldots,  i_t)$ includes the actual state $s_t$ of the system together with the interference labels which causally affects all $s'_t$, $Q_t$, and $i_t$.
Note that $z_t$ is not observable to the agent due to the interference; $z_t$ could be viewed as a hidden confounder in causal inference.

Since only the interfered observation $s'_t$ is available, the interference label $i_t$ is also non-observable in evaluating the resilience ability of an agent. 
However, the interference information is often accessible in the training phase, such as the use of a navigation simulator recorded with noisy augmentation~\citep{grigorescu2020survey}~for simulating interference in the training environment. 
We will discuss in the next subsection the benefit of utilizing the interference labels to improve learning efficiency.

\subsection{Learning with Interference Labels}

The goal of a resilient RL agent is to learn $P(Q_t | s'_t)$ to infer the Q-value $Q_t$ based on the interfered observation $s'_t$.
Note that one can compute $P(Q_t | s'_t)$ by determining the joint distribution $P(z_t, s'_t, i_t, Q_t)$ of all variables in the CGM in Figure~\ref{fig:diagram}. 
Despite the presence of the hidden variable $z_t$, 
similar to causal inference with hidden confounders~\citep{louizos2017causal}, estimating the joint distribution $P(z_t, s'_t, i_t, Q_t)$ could be done efficiently when the agent is provided the interference labels $i_t$ during training.
On the other hand, if only the observation $s'_t$ is available, the agent can only directly estimate $P(Q_t|s'_t)$, which is less efficient in terms of training sample usage.

We provide the interference type $\mathcal I$ and the interference labels $i_t$ to efficiently train a resilient RL agent as shown in Figure~\ref{fig:diagram}(b); however, in the actual testing environment, the agent only has access to the interfered observations $x'_t$ as in Figure~\ref{fig:diagram}(a).

\begin{figure}[ht!]
	\centering
	\begin{subfigure}{0.10\textwidth} %
	    \centering
	\begin{tikzpicture}[x = 2.0cm, y=2.0cm, scale=0.8, every node/.style={scale=0.8}]
\tikzstyle{var} = [draw, circle, minimum height=1cm,text centered, line width=0.5pt ]
\tikzstyle{obs} = [var]
\tikzstyle{train} = [var]

\node[var] at (0,0) (zt) {$z_t$};
\node[train, fill=gray!30] at (0.8, 0.8) (it) {$i_t$};
\node[obs, fill=gray!30] at (0, 0.8) (txt) {$s'_t$};
\node[obs, fill=gray!30] at (0.8, 0) (rt) {$Q_t$};

\draw[arrow] (zt) -- (it);
\draw[arrow] (zt) -- (txt);
\draw[arrow] (zt) -- (rt);
\end{tikzpicture}
\caption{\small Training}
\end{subfigure}
\quad\quad\quad
\begin{subfigure}{0.10\textwidth} %
	    \centering
	\begin{tikzpicture}[x = 2.0cm, y=2.0cm, scale=0.8, every node/.style={scale=0.8}]
\tikzstyle{var} = [draw, circle, minimum height=1cm,text centered, line width=0.5pt ]
\tikzstyle{obs} = [var]
\tikzstyle{train} = [var]

\node[var] at (0,0) (zt) {$z_t$};
\node[train] at (0.8, 0.8) (it) {$i_t$};
\node[obs, fill=gray!30] at (0, 0.8) (txt) {$s'_t$};
\node[obs, fill=gray!30] at (0.8, 0) (rt) {$Q_t$};

\draw[arrow] (zt) -- (it);
\draw[arrow] (zt) -- (txt);
\draw[arrow] (zt) -- (rt);
\end{tikzpicture}
\caption{\small Testing}
\end{subfigure}

\caption{ Causal graphical model (CGM) for the training phase (a) and the testing phase (b).
White nodes $s'_t$ and $Q_t$ are observable. Node $z_t=(s_t, i_{t-M+1}, \ldots,  i_t)$, colored by white, is not observable. Node $i_t$, colored by white in (b), is only observable during training.
}

\label{fig:diagram}
\end{figure}
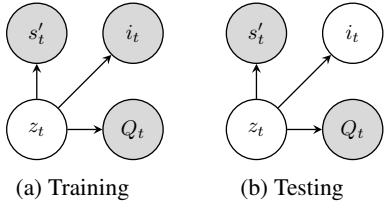
\subsection{Causal Inference Q-Network}

With the observable variables $(s'_t, i_t, Q_t)$ in Figure~\ref{fig:diagram}(a) during training, we aim to learn a model to infer the Q-values by estimating the joint distribution $P(z_t, s'_t, i_t, Q_t)$.
Despite the underlying dynamics in the RL system, when we view the interference as a treatment, the CGM in Figure~\ref{fig:diagram}(a) resembles some common causal inference models with binary treatment information and hidden confounders~\citep{louizos2017causal}. 
In this kind of causal inference problems, 
by leveraging on the binary property for treatment information, TARNet~\citep{shalit2017estimating} and CEVAE~\citep{louizos2017causal} introduced a binary switching neural architecture to efficiently learn latent models for causal inference. 

Inspired by the switching mechanism for causal inference, we propose the causal inference Q-network, referred as CIQ, that maps the interfered observation $s'_t$ into a latent state $z_t$, makes proper inferences about the interference condition $i_t$, and adjusts its policy based on the estimated interference. 

\begin{figure}[ht!]
\centering
\includegraphics[width=0.48\textwidth]{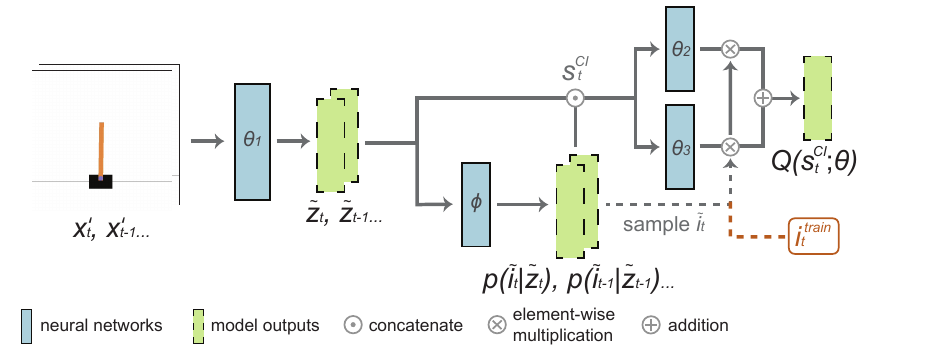}
\caption{ CIQ architecture. The notation $i^{train}_t$ denotes the inference label available during training, whereas $\tilde i_t$ is sampled during inference as $i_t$ is unknown.}
\label{Nets}
\end{figure}

We approximate the latent state by a neural network $\tilde z_t = f_1(x'_t; \theta_1)$. From the latent state, we generate the estimated interference label $\tilde i_t \sim p(\tilde i_t|z_t) = f_I(z_t; \phi)$. We denote $s^{CI}_t = (\tilde z_{t-M+1}, \tilde i_{t-M+1}, \ldots, \tilde z_t, \tilde i_t)$ to be the causal inference state.
As discussed in the previous subsection, the causal inference state acts as a confounder between the interference and the reward.
Therefore, instead of using the interfered state $s'_t$, the causal inference state $s^{CI}_t$ contains more relevant information for the agent to maximize rewards.
Using the causal inference state helps focus on meaningful and informative details even under interference.

With the causal inference state $s^{CI}_t$, the output of the Q-network $Q(s^{CI}_t; \theta)$ is set to be switched between two neural networks $f_2(s^{CI}_t;\theta_2)$ and $f_3(s^{CI}_t; \theta_3)$ by the interference label.
Such a switching mechanism prevents our network from over-generalizing the causal inference state. During training, switching between the two neural networks is determined by the training interference label $i_t^{\text{train}}$. We assume that the true interference label is available in the training phase so $i_t^{\text{train}} = i_t$.
In the testing, when $i_t$ is not available, we use the predicted interference label $\tilde i_t$ as the switch to decide which of the two neural networks to use.

All the neural networks $f_1, f_2, f_3, f_I$ have two fully connected layers\footnote{Though such manner may lead to the myth of over-parameterization, our ablation study proves that we can achieve better results with almost the same amount of parameters.} 
with each layer followed by the ReLU activation except for the last layer in $f_2, f_3$ and $f_I$. The overall CIQ model is shown in Figure \ref{Nets} and $\theta = (\theta_1, \theta_2, \theta_3, \phi)$ denotes all its parameters.
Note that, as common practice for discrete action spaces, the Q-network output $Q(s^{CI}_t; \theta)$ is an $\mathcal A$-dimensional vector where $\mathcal A$ is the size of the action space, and each dimension represents the value for taking the corresponding action.

Finally, we train the CIQ model $Q(s'_t; \theta)$ end-to-end by the DQN algorithm with an additional loss for predicting the interference label.
The overall CIQ objective function is defined as: 
\begin{align}
\label{eq:loss:3}    
    &L^{\text{CIQ}}(\theta_1, \theta_2, \theta_3, \phi)   
     = i_t^{\text{train}} \cdot L^{\text{DQN}}(\theta_1, \theta_2, \phi)
    \nonumber \\
     &+(1 - i_t^{\text{train}}) \cdot L^{\text{DQN}}(\theta_1, \theta_3, \phi) 
    + \lambda \cdot  (i_t^{\text{train}} \log p(\tilde i_t | \tilde z_t ; \theta_1, \phi)\nonumber \\&+(1 - i_t^{\text{train}}) \log(1 - p(\tilde i_t | \tilde z_t ; \theta_1, \phi))),
\end{align}
where $\lambda$ is a scaling constant and is set to 1 for simplicity. Due to the design of the causal inference state and the switching mechanism, we will show that CIQ can perform resilient behaviors against the observation interferences. We introduce how to quantify the robustness of a Q-network under noisy observation in next subsection. The CIQ training procedure (Algorithm \ref{CIA_training}) and an advanced CIQ based on variational inference~\citep{louizos2017causal} are described in Appendix B.

\subsection{CLEVER-Q: A Robustness Evaluation Metric for Q-Networks}

Here we provide a comprehensive score (CLEVER-Q) for evaluating the robustness of a Q-network model by extending the CLEVER robustness score ~\citep{weng2018evaluating} designed for classification tasks to Q-network based DRL tasks. Consider an $\ell_p$-norm bounded ($p \geq 1$) perturbation $\delta$ to the state $s_t$. We first derive a lower bound $\beta_L$ on the minimal perturbation to $s_t$ for altering the action with the top Q-value, i.e., the greedy action. For a given $s_t$ and a Q-network, this lower bound $\beta_L$ provides a robustness guarantee that the greedy action at $s_t$ will be the same as that of \textit{any} perturbed state $s_t+\delta$, as long as the perturbation level $\|\delta\|_p \leq \beta_L$. Therefore, the larger the value $\beta_L$ is, the more resilience of the Q-network against perturbations can be guaranteed. Our CLEVER-Q score uses the extreme value theory to evaluate the lower bound  $\beta_L$ as a robustness metric for benchmarking different Q-network models. The proof of Theorem 1. is available in Appendix B. %

\begin{theorem}

Consider a Q-network $Q(s,a)$ and a state $s_t$. Let $\mathcal{A}^* = \argmax_{a} Q(s_t, a)$ be the set of greedy (best) actions having the highest Q-value at $s_t$ according to the Q-network. Define $g_a(s_t) = Q(s_t, \mathcal{A}^*) - Q(s_t, a)$ for 
every action $a$, where $Q(s_t, \mathcal{A}^*)$ denotes the best Q-value at $s_t$.
Assume $g_a(s_t)$ is locally Lipschitz continuous\footnote{Here locally Lipschitz continuous means $g_a(s_t)$ is Lipschitz continuous within the $\ell_p$ ball centered at $s_t$ with radius $R_p$. We follow the same definition as in \citep{weng2018evaluating}.} with its local Lipschitz constant denoted by $L_q^a$, where $1/p+1/q = 1$ and $p \geq 1$. For any $p\geq 1$, define the lower bound
\begin{equation}
    \beta_{L}= min_{a \notin  \mathcal{A}^*} g_a(s_t) / L_q^a.
    \label{eq:cleverq}
\end{equation}
Then for any $\delta $ such that $\|\delta\|_p \leq \beta_L$, we have
   $\argmax_{a} Q(s_t, a) = \argmax_{a} Q(s_t + \delta, a)$.
\end{theorem}

\begin{figure}[ht]
\begin{center}
   \includegraphics[width=0.75\linewidth]{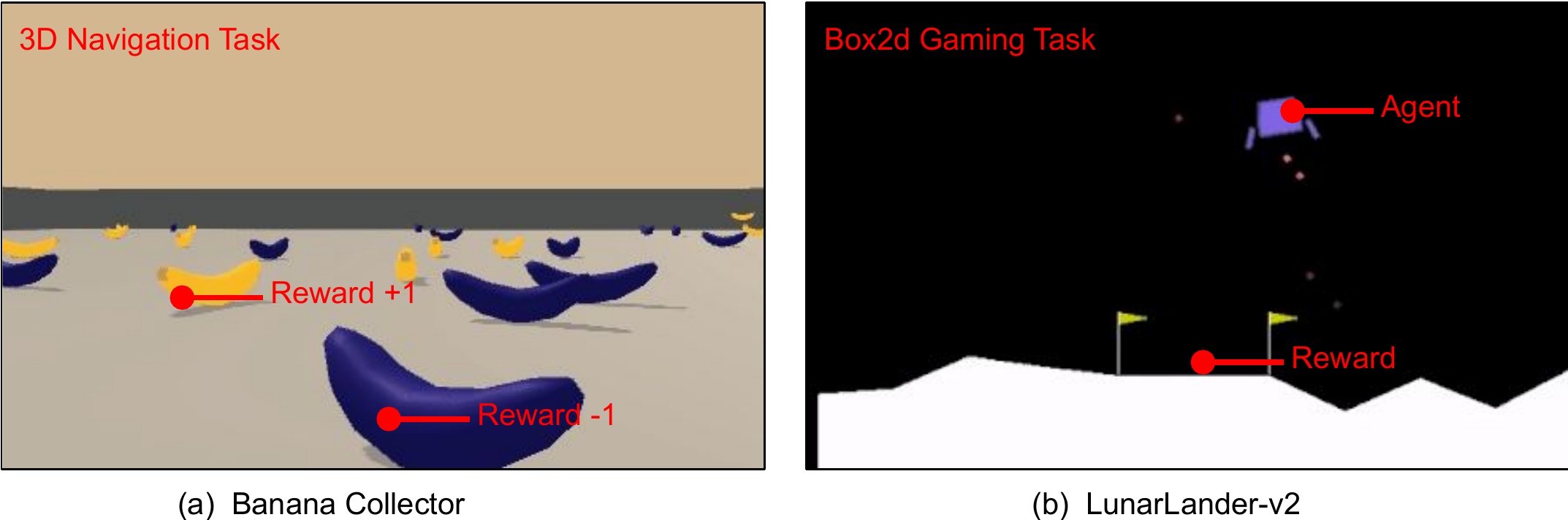}
\end{center}
\vspace{-0.4cm}
\caption{ Illustration of our environments on: (a) a 3D navigation task, banana collector~\citep{juliani2018unity}, and (b) a video game, LunarLander~\citep{brockman2016openai}. } 
\label{fig:new:env}
\end{figure}
\section{Experiments}
\subsection{Environments for DQNs}
Our testing platforms were based on (a) OpenAI Gym \citep{brockman2016openai}, (b) Unity-3D environments \citep{juliani2018unity}, (c) a 2D gaming environment~\citep{brockman2016openai}, and (d) visual learning from pixel inputs of cart pole. Our test environments cover some major application scenarios and feature discrete actions for training DQN agents with the CLEVER-Q analysis. For instance, Atari games and space-invaders are popular real-world applications. Unity 3D banana navigation is a physical simulator but provides virtual to real options for further implementations.

\textbf{Vector Cartpole:}  Cartpole~\citep{sutton1998reinforcement} is a classical continuous control problem. We use Cartpole-v0 from Gym~\citep{brockman2016openai} with a targeted reward  $=195.0$. The defined environment is manipulated by adding a force of $+1$ or $-1$ to a moving cart. 

\textbf{Banana Collector:} The Banana collector shown in Figure \ref{fig:new:env} (a) is one of the Unity 3D baseline~\citep{juliani2018unity}. Different from the MuJoCo simulators with continuous actions, the Banana collector is controlled by four discrete actions corresponding to moving directions. The targeted reward is $12.0$ points by accessing correct bananas ($+1$). The state-space has 37 dimensions included velocity and a ray-based perception of objects around the agent.

\textbf{Lunar Lander:} Similar to the Atari gaming environments, Lunar Lander-v2 (Figure \ref{fig:new:env} (c)) is a discrete action environment from OpenAI Gym~\citep{brockman2016openai} to control firing ejector with a targeted reward of $200$. The state is an eight-dimensional vector that records the lander’s position, velocity, angle, and angular velocities. The episode finishes if the lander crashes or comes to rest, receiving a reward $-100$ or $+100$ Firing ejector costs $-0.3$ each frame with $+10$ for each ground contact.

\textbf{Pixel Cartpole:}
To further evaluate our models, we conduct experiments from the pixel inputs in the cartpole environment as a visual learning task. The size of input state is $400\times600$. We use a max-pooling and a  convolution layer to extract states as network inputs. The environment includes two discrete actions $\left \{ left, right \right \}$, which is identical to the Cartpole-v0 of the vector version. 

\subsection{Baseline Methods}
In the experiments, we compare our CIQ algorithm with two sets of DQN-based DRL baselines to demonstrate the resilience capability of the proposed method. We ensure all the models have the \textbf{same number} of 9.7 millions \textbf{parameters} with careful fine-tuning to avoid model capacity issues.

\textbf{Pure DQN:} We use DQN as a baseline in our experiments. The DQN agent is trained and tested on interfered state $s'_t$. We also evaluate common DQN improvements in Appendix C and find the improvements (e.g., DDQN) have no significant effect against interference.

\textbf{DQN with an interference classifier (DQN-CF):} 
In the resilient reinforcement learning framework, the agent is given the true interference label $i_t^{\text{train}}$ at training.
Therefore, we would like to provide this additional information to the DQN agent for a \textbf{fair comparison.}
During training, the interfered state $s'_t$ is concatenated with the true label $i_t^{\text{train}}$ as the input for the DQN agent.
Since the true label is not available at testing, we train an additional binary classifier (CF) for the DQN agent.
The classifier is trained to predict the interference label, and this predicted label will be concatenated with the interfered state as the input for the DQN agent during testing.

\textbf{DQN with safe actions (DQN-SA):} 
Inspired by shielding-based safe RL~\citep{alshiekh2018safe}, we consider a DQN baseline with safe actions (SA). The DQN-SA agent will apply the DQN action if there is no interference. However, if the current observation is interfered, it will choose the action used for the last uninterfered observation as the safe action. This action-holding method is also a typical control approach when there are missing observations \citep{franklin1998digital}.
Similar to DQN-CF, a binary classifier for interference is trained to provide predicted labels at testing.

\textbf{DVRLQ and DVRLQ-CF:} Motivated by deep variational RL (DVRL)~\citep{igl2018deep}, we provide a version of DVRL as a POMDP baseline. We call this baseline DVRLQ because we replace the A2C-loss with the DQN loss. Similar to DQN-CF, we also consider another baseline of DVRLQ with a classifier, referred to as DVRLQ-CF, for a fair comparison using the interference labels.
\begin{figure}[ht!]
	\begin{subfigure}{0.40\textwidth} %
	    \centering
	    \includegraphics[width=\textwidth]{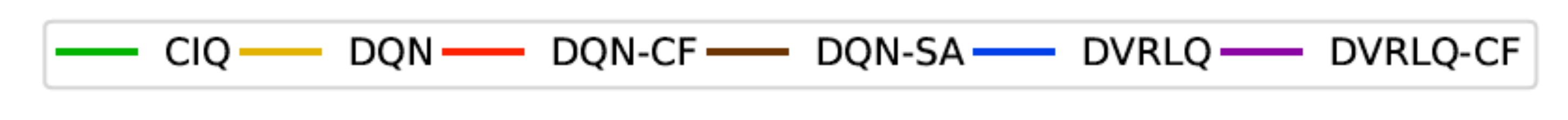}
	\end{subfigure}
	\quad
	\begin{subfigure}{0.22\textwidth} %
	    \centering
	    \includegraphics[width=\textwidth]{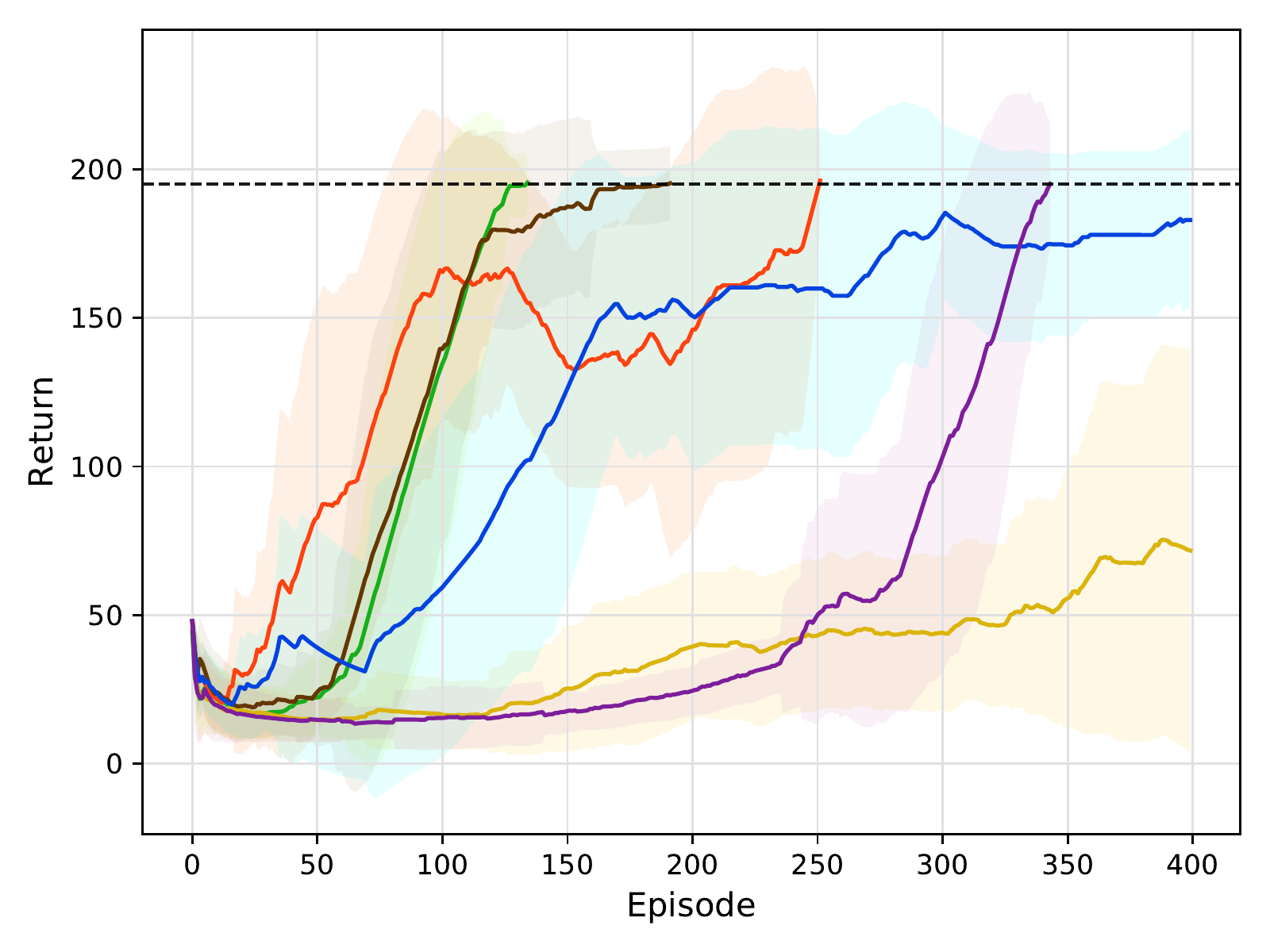}
	    \caption{Cartpole$^{adversarial}_{vector}$.} %
	\end{subfigure}
	\quad
	\begin{subfigure}{0.22\textwidth} %
	    \centering
		\includegraphics[width=\textwidth]{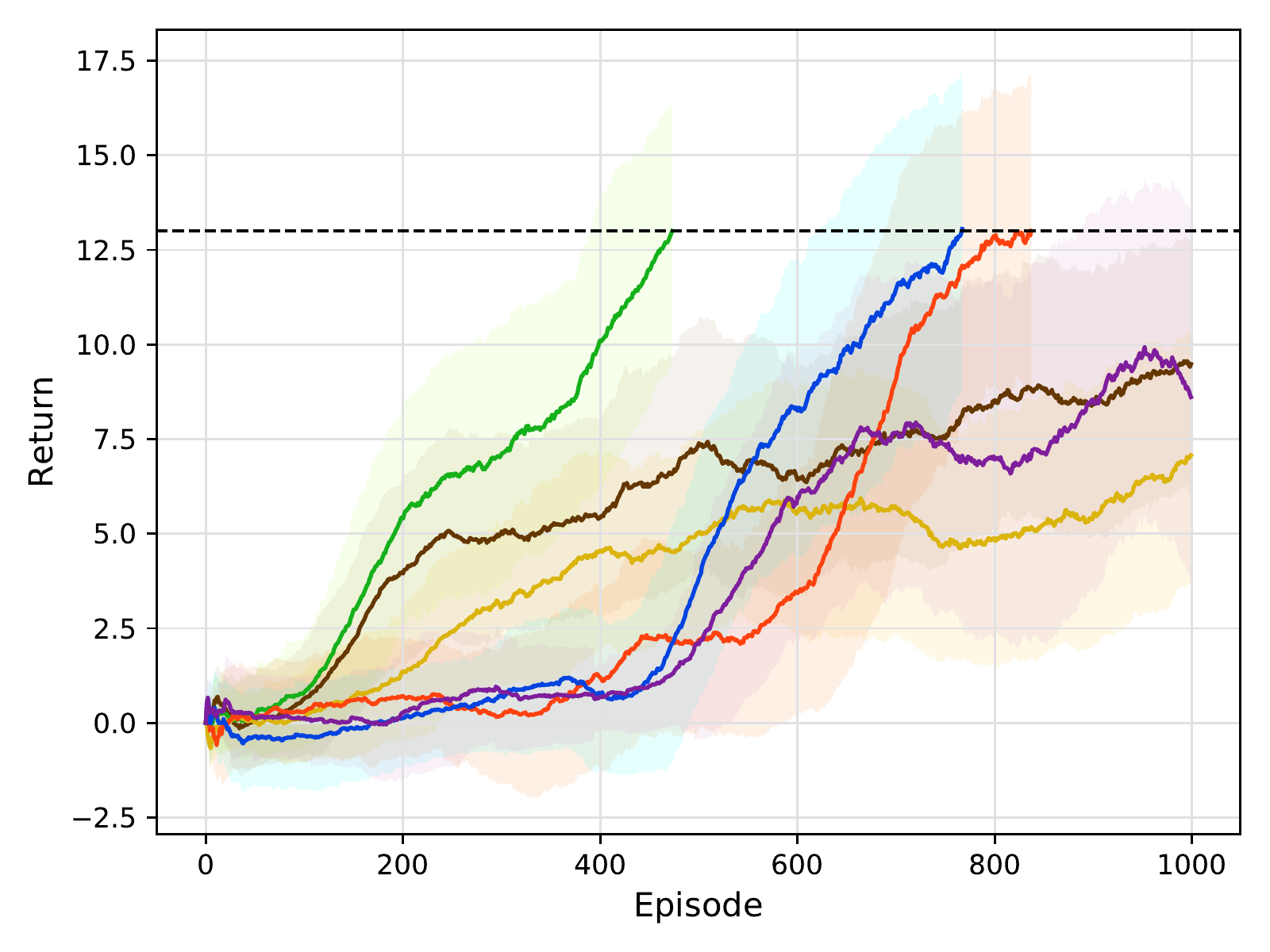}
		\caption{Banana$^{adversarial}$. } %
	\end{subfigure}
	\quad
	\begin{subfigure}{0.22\textwidth} %
	    \centering
		\includegraphics[width=\textwidth]{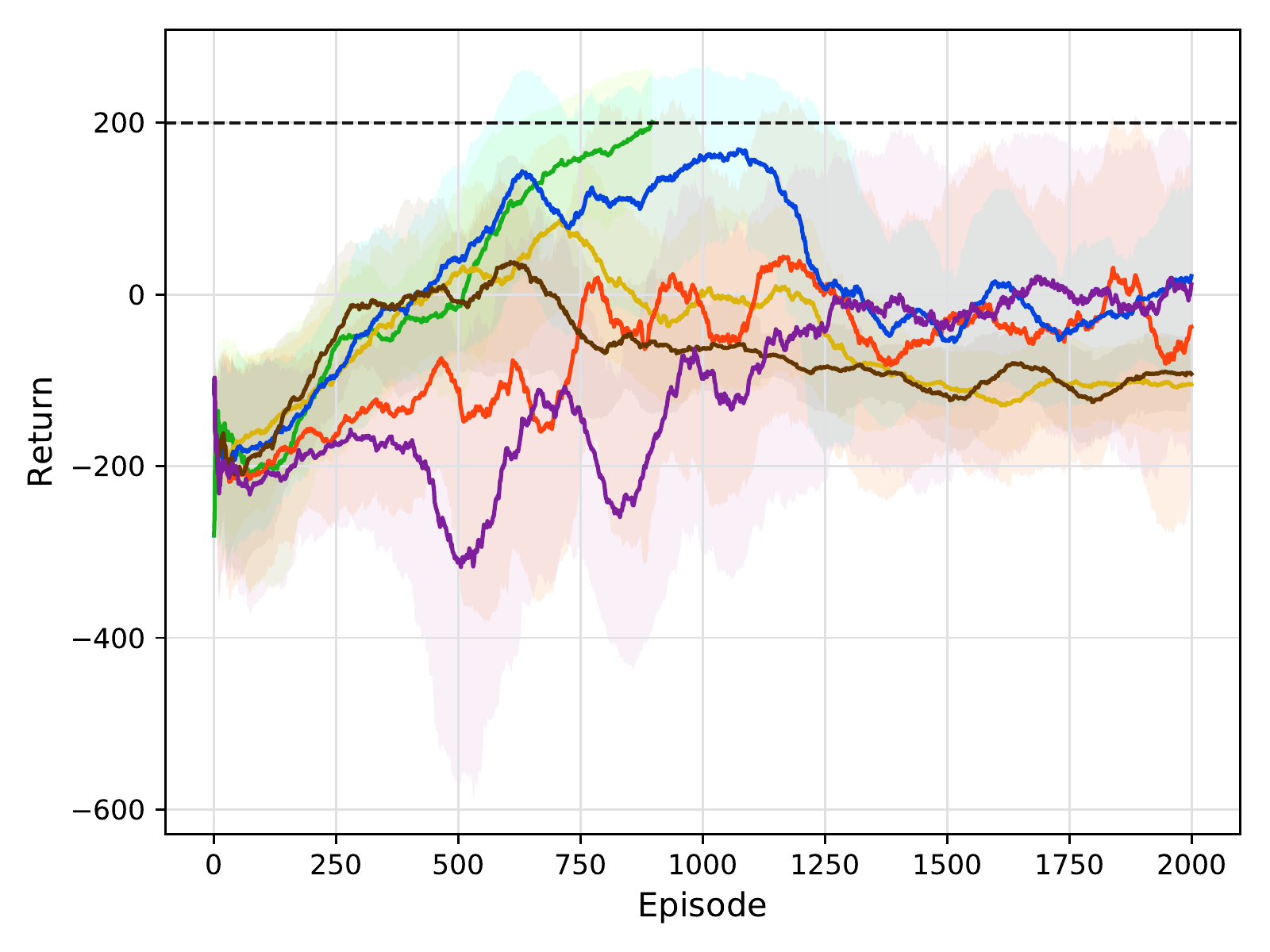}
		\caption{Lunar$^{adversarial}$.} %
	\end{subfigure}
	\quad %
	\begin{subfigure}{0.22\textwidth} %
	    \centering
		\includegraphics[width=\textwidth]{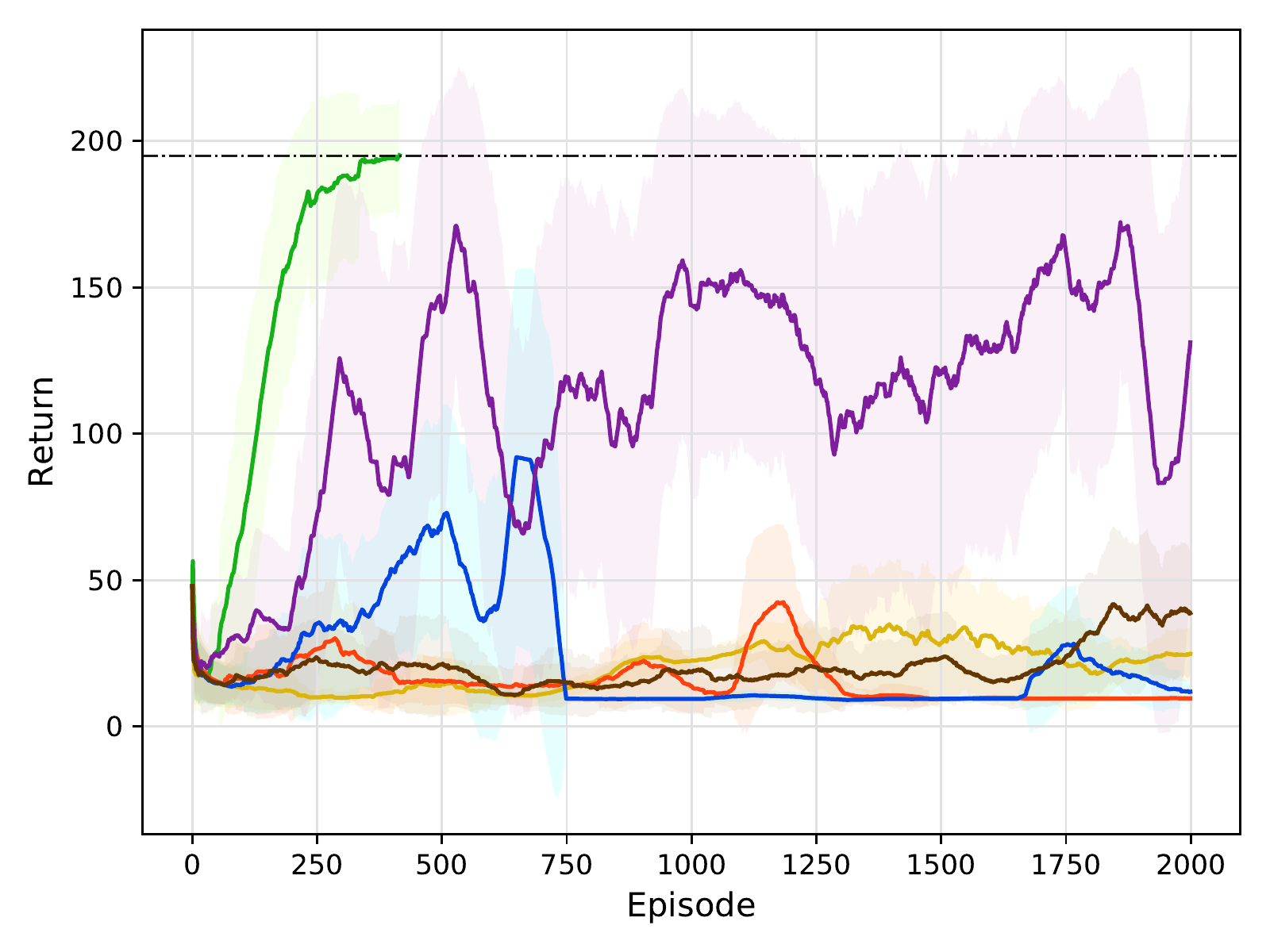}
		\caption{Cartpole$^{adversarial}_{pixel}$.} %
		\end{subfigure}
    \quad
	 \begin{subfigure}{0.22\textwidth} %
	    \centering
	    \includegraphics[width=\textwidth]{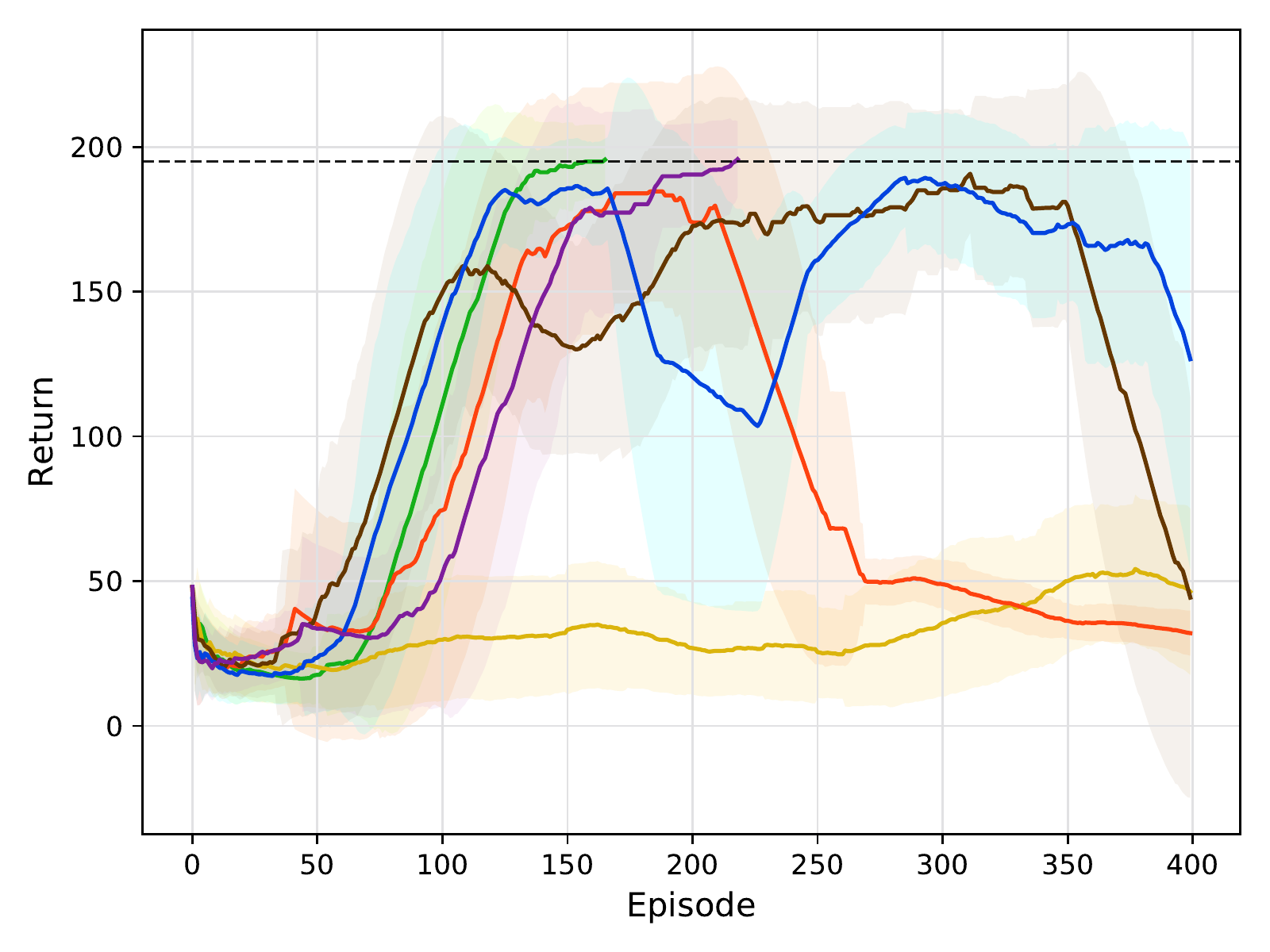}
	    \caption{Cartpole$^{blackout}_{vector}$.} %
	\end{subfigure}
	\quad~
	\begin{subfigure}{0.22\textwidth} %
	    \centering
		\includegraphics[width=\textwidth]{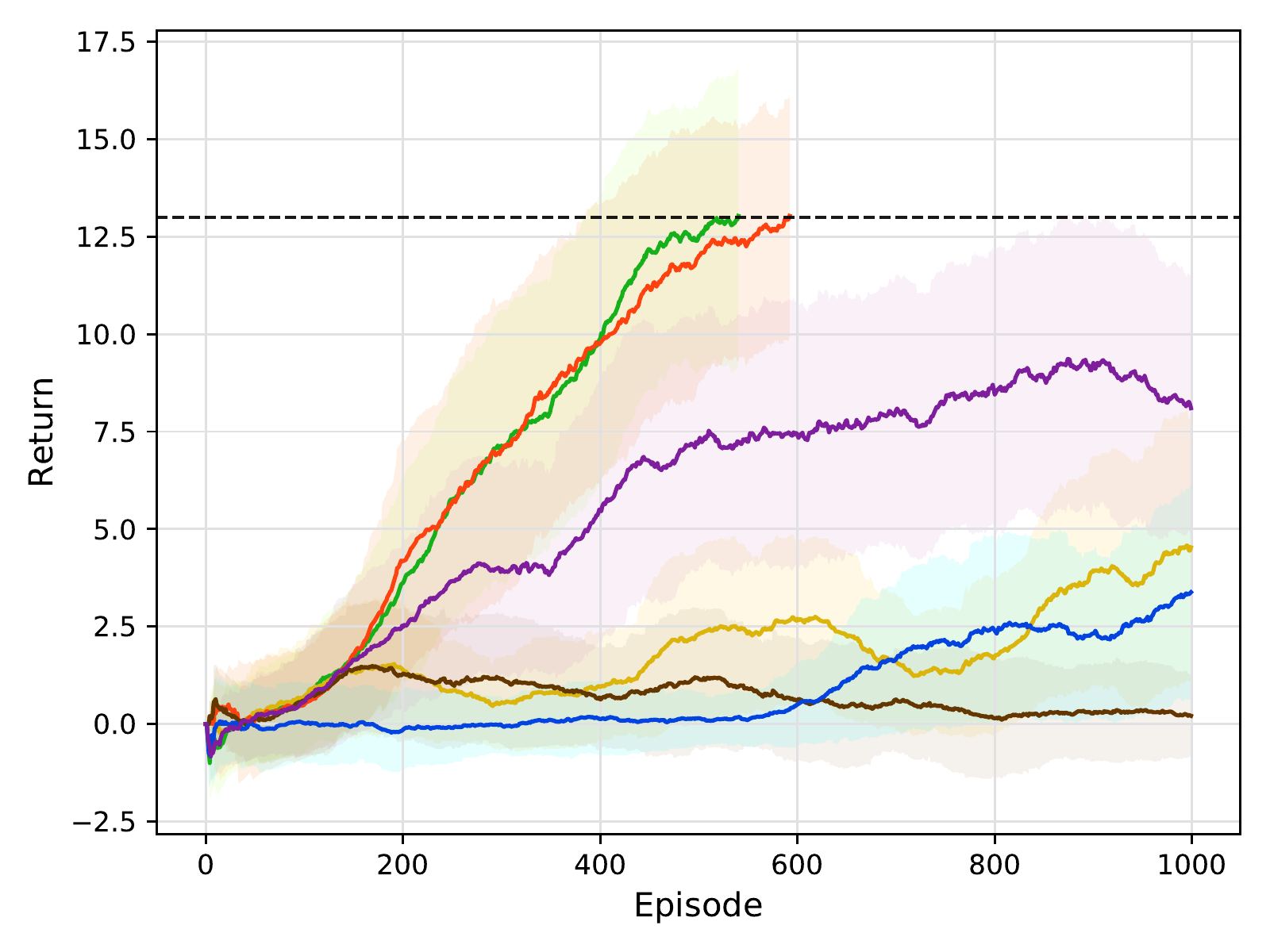}
		\caption{Banana$^{blackout}$. } %
	\end{subfigure}
	\quad~
	\begin{subfigure}{0.22\textwidth} %
	    \centering
		\includegraphics[width=\textwidth]{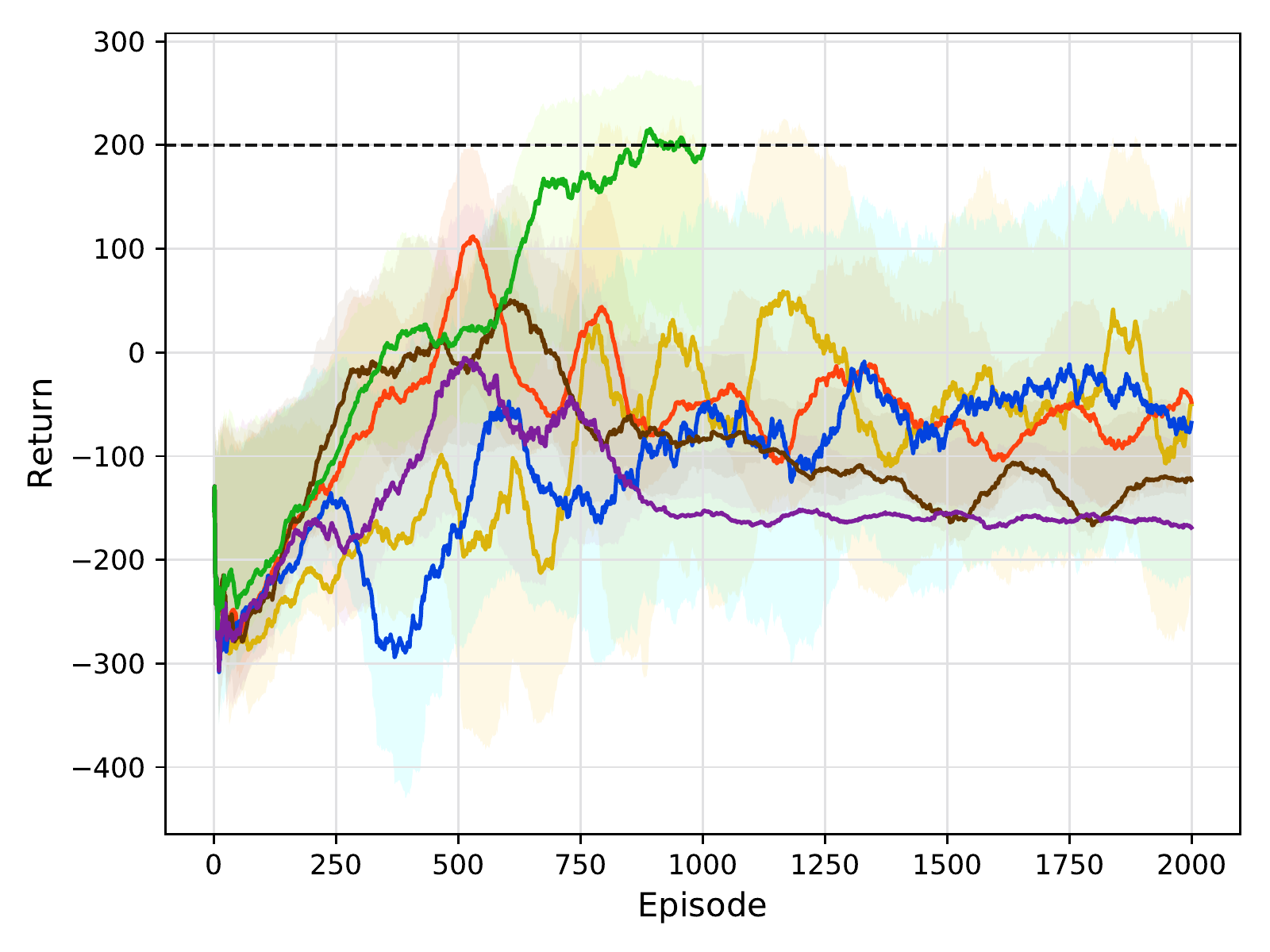}
		\caption{Lunar$^{blackout}$} %
	\end{subfigure}
	\quad~
	\begin{subfigure}{0.22\textwidth} %
	    \centering
		\includegraphics[width=\textwidth]{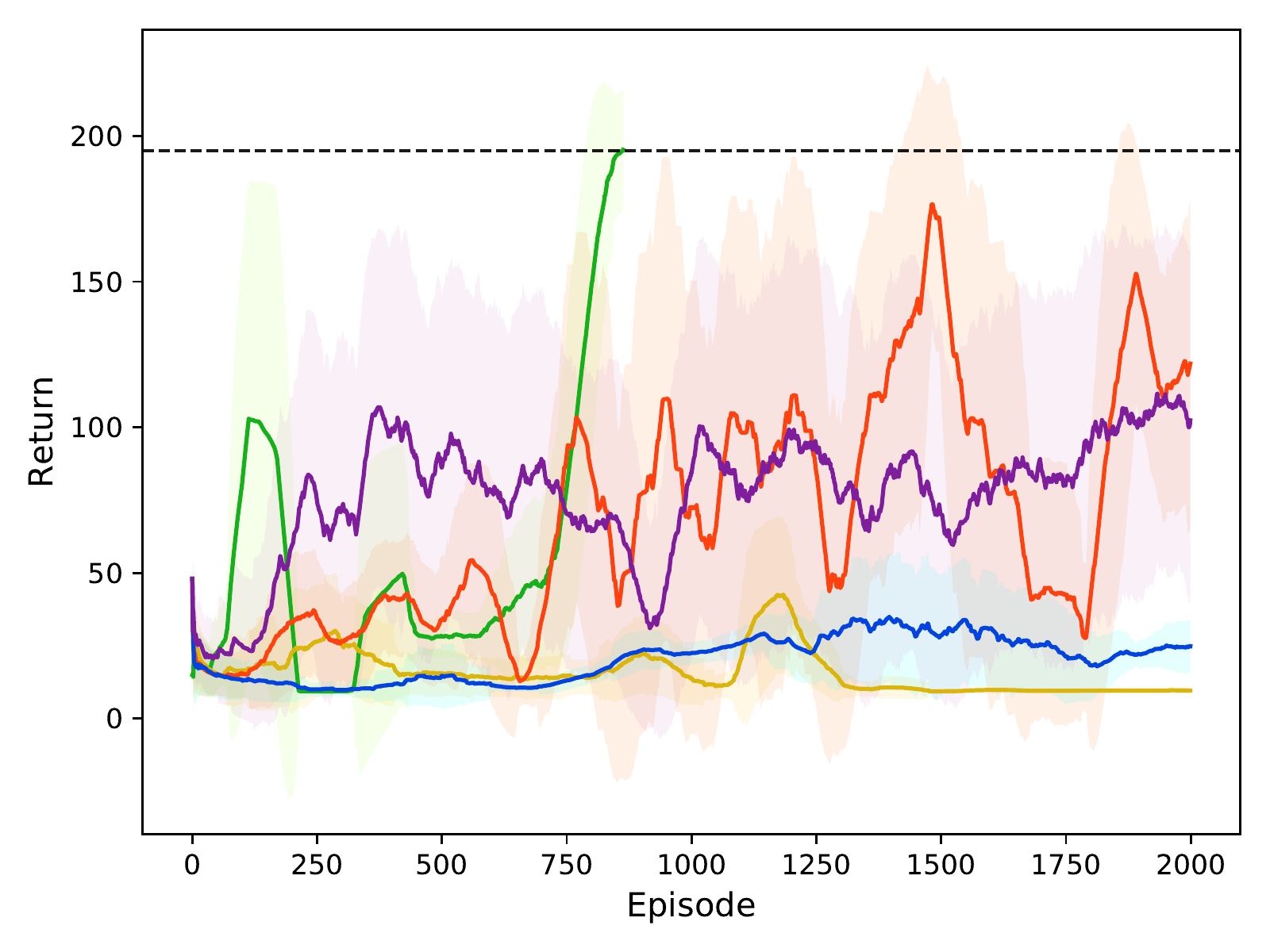}
		\caption{Cartpole$^{blackout}_{pixel}$.} %
	\end{subfigure}
	\caption{Performance of DQNs under potential (20\%) adversarial and black-out interference. } %
	\label{fig:figure:7:adv}
\end{figure}

\begin{table*}[ht!]
\centering
\caption{Performance resilience analysis of AC-Rate ($\uparrow$) and CLEVER-Q robustness score ($\uparrow$) under additive Gaussian ($l_{2}$-norm) and adversarial ($l_{\infty}$-norm) perturbations on state in the vector Cartpole environment. 
}
\begin{tabular}{|c|cc|cc|c|cc|cc|}
\hline
$\mathcal I$=L$_2$            & \multicolumn{2}{c|}{AC-Rate} & \multicolumn{2}{c|}{CLEVER-Q} & $\mathcal I$=L$_{\infty}$            & \multicolumn{2}{c|}{AC-Rate} & \multicolumn{2}{c|}{CLEVER-Q} \\ \hline
P\%,~$\mathcal I$ & DQN               & CIQ               & DQN           & CIQ           &  P\%,~$\mathcal I$ & DQN               & CIQ               & DQN           & CIQ           \\ \hline
10\%           & 82.10\%           & \textbf{99.61\%}           & 0.176         & \textbf{0.221}         &  10\%           & 62.23\%           & \textbf{99.52\% }          & 0.169         & \textbf{0.248}         \\ \cline{1-1} \cline{6-6}
20\%            & 72.15\%           & \textbf{98.52\%}           & 0.130         & \textbf{0.235}         &  20\%           & 9.68\%            & \textbf{98.52\%}           & 0.171         & \textbf{0.236}         \\ \cline{1-1} \cline{6-6}
30\%            & 69.74\%           & \textbf{98.12\%}           & 0.109         & \textbf{0.232}         &  30\%           & 1.22\%            & \textbf{98.10\%}           & 0.052         & \textbf{0.230}         \\ \hline
\end{tabular}
\label{tab:table:1}
\end{table*}
\subsection{Resilient RL on Average Returns}

We run performance evaluation with six different interference probabilities ($p^{I}$ in Sec. \ref{sub_41}), including $\left \{ 0\%, 10\%, 20\%, 30\%, 40\%, 50\% \right \}$. We train each agent $50$ times and highlight its standard deviation with lighter colors. Each agent is trained until the target score (shown as the dashed black line) is reached or until 400 episodes. We show the average returns for $p^{I}=20\%$ under adversarial perturbation and black-out in Figure \ref{fig:figure:7:adv} and report the rest of the results in Appendix B. %

CIQ (\textcolor{teal}{green}) clearly outperforms all the baselines under all types of interference, validating the effectiveness of our CIQ in learning to infer and gaining resilience against a wide range of observational interferences.
Pure DQN (yellow) cannot handle the interference with $20\%$ noise level.
DQN-CF (orange) and DQN-SA (brown) have competitive performance in some environments against certain interferences, but perform poorly in others. DVRLQ (blue) and DVRLQ-CF (purple) cannot achieve the target reward in most experiments and this might suggest the inefficiency  of applying a general POMDP approach in a framework with a specific structure of observational interference.

\subsection{Robustness Metrics based on Recording States}

We evaluate the robustness of DQN and CIQ by the proposed CLEVER-Q metric.
To make the test state environment consistent among different types and levels of interference, we record the interfered states, $S_{N} = \mathcal I(S_{C})$, together with their clean states, $S_{C}$. We then calculate the average CLEVER-Q for DQN and CIQ based on the clean states $S_{C}$ using Eq. \ref{eq:cleverq} over 50 times experiments for each agent.

We also consider a retrospective robustness metric, the action correction rate (AC-Rate).
Motivated by previous off-policy and error correction studies~\citep{dulac2012fast, harutyunyan2016q, lin2017tactics}, AC-Rate is defined as the action matching rate $R_{Act} = \frac{1}{T} \sum_{t=0}^{T-1} \mathbf{1}_{\{a_t = a^*_t\}}$ between $a_t$ and $a^*_t$ over an episode with length $T$. 
Here $a_t$ denotes the action taken by the agent with interfered observations $S_{N}$, and $a^*_t$ is the action of the agent if clean states $S_{C}$ were observed instead.

The roles of CLEVER-Q and AC-Rate are complementary as robustness metrics. CLEVER-Q measures sensitivity in terms of the margin (minimum perturbation) required for a given state to change the original action. AC-rate measures the utility in terms of action consistency. Altogether, they provide a comprehensive resilience assessment.

Table \ref{tab:table:1} reports the two robustness metrics for DQN and CIQ under two types of interference.  CIQ attains higher scores than DQN in both CLEVER-Q and AC-Rate, reflecting better resilience in CIQ evaluations. We provide more robustness measurements in Appendix A and D.

\subsection{Average Treatment Effect under Intervention}
\label{sup:d:causal:effect}
In a causal learning setting, evaluating treatment effects and conducting statistical refuting experiments are essential to support the underlying causal graphical model. Through resilient reinforcement learning framework, we could interpret DQN by estimating the average treatment effect (ATE) of each noisy and adversarial observation. 
We first define how to 
calculate a treatment effect in the resilient RL settings and conduct statistical refuting tests including random common cause variable test ($T_{c}$), replacing treatment with a random (placebo) variable ($T_{p}$), and removing a random subset of data ($T_{s}$). The open-source causal inference package Dowhy~\citep{sharma2019dowhy} is used for analysis.

We refine a Q-network with discrete actions for estimating treatment effects based on Theorem 1 in~\citep{louizos2017causal}.
In particular, individual treatment effect (ITE) can be defined as the difference between the two potential outcomes of a Q-network; and the average treatment effect (ATE) is the expected value of the potential outcomes over the subjects. 
In a binary treatment setting, for a Q-value function $Q_t(s_t)$ and the interfered state $\mathcal I(s_t)$, the ITE and ATE are calculated by:
\begin{align}
    Q_{t}^{ITE}=Q_{t}(s_t)\left(1-p_{t}\right)+Q_{t}({\mathcal I}(s_{t})) p_{t} 
   \\ 
   A T E= \sum_{t=1}^{\mathcal T}\frac{\mathbb{E}\left[Q_{t}^{ITE}({\mathcal I}(s_t))\right]-\mathbb{E}\left[Q_{t}^{ITE}(s_t)\right]}{\mathcal T}
\end{align}
where $p_t$ is the estimated inference label by the agent and $\mathcal T$ is the total time steps of each episode. As expected, we find that CIQ indeed attains a better ATE and its significance can be informed by the refuting tests based on
$T_c$, $T_{p}$ and $T_s$. We refer to Appendix C for more details.

\subsection{Additional analysis}
We also conduct the following analysis to better understand our CIQ model. Environments with a dynamic noise level are evaluated. Due to the space limit, see their details in appendix B to D. Furthermore, a discussion on the advantage of\textbf{ sample complexity benefited from sequential learning with interference labels} is included in Appendix B.

\textbf{Neural saliency map:} We apply the perturbation-based saliency map for DRL~\citep{greydanus2018visualizing} as shown in Figure \ref{fig:figure:4} and appendix \ref{sup:e:saliency:map}to visualize the saliency centers of CIQ and others, which is based on the Q-value of each model as interpretable studies.

\textbf{Treatment effect analysis:} We provide treatment effect analysis on each kind of interference to statistically verify the CGM with lowest errors on average treatment effect refutation in appendix C.

\textbf{Ablation studies:} We conduct ablation studies by comparing several CIQ variants, each without a certain CIQ component, and verify the importance of the proposed CIQ architecture in Appendix D for future studies.

\textbf{Test on different noise levels:} We train CIQ under one noise level and test on another level, which shows that the difference in noise level does not affect much on the performance of CIQ model reported in Appendix B.

\textbf{Transferability in robustness:} Based on CIQ, we study how well can the robustness of different interference types transfer between training and testing environments. We evaluate two general settings  (i) an identical interference type but different noise levels (Appendix D) and (ii) different interference types (Appendix D). Tab~\ref{tab:CIQ:MI} summarizes the results. 

\textbf{Multiple interference types: } We also provide a generalized version of CIQ that deals with multiple interference types in training and testing environments. Tab~\ref{tab:diffn:m} summarizes the results. 
The generalized CIQ is 
equipped with a common encoder and individual interference decoders to study multi-module conditional inference, with some additional discussion in Appendix E.

\begin{figure}[thb!]
\begin{center}
   \includegraphics[width=0.96\linewidth]{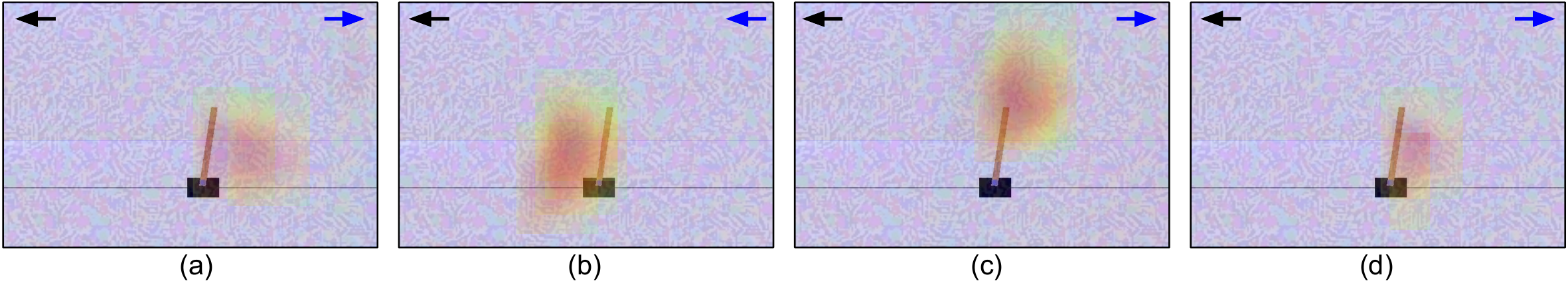}
\end{center}
\vspace{-0.4cm}
\caption{Perturbation-based saliency map on Pixel Cartpole under adversarial perturbation: (a) DQN, (b) CIQ, (c) DQN-CF; (d) DVRLQ-CF. The black arrows are correct actions and blue arrows are agents' actions. The neural saliency of CIQ makes more correct actions responding to ground actions.} 
\label{fig:figure:4}
\vspace{-0.2cm}
\end{figure}

\begin{table}[ht!]
\centering
\caption{Stability test of proposed CIQ (\emph{Train} Noise-Level, \emph{Test} Noise-Level). We consider settings with different training and testing noise levels for CIQ evaluation afterward. }
\begin{adjustbox}{width=0.47\textwidth}
\begin{tabular}{|l|llllll|}
\hline
Metrics & (0.1, 0.3) & (0.3, 0.1) & (0.3, 0.2) & (0.3, 0.3) & (0.3, 0.4) & (0.3, 0.5) \\ \hline
Performance & 182.8 &\textbf{195.0}& \textbf{195.0}& \textbf{195.0} &\textbf{195.0} & 185.7 \\ \hline
CLEVER-Q & 0.195 & 0.239 & 0.232 & 0.230 & 0.224 & 0.215  \\ \hline
AC-Rate & 0.914 & 0.985 & 0.986 & 0.995 & 0.984 & 0.924 \\ \hline
\end{tabular}
\end{adjustbox}
\label{tab:diffn:m}
\vspace{-0.3cm}
\end{table}

\begin{table}[ht!]
\centering
\caption{CIQ-MI: CIQ agent with an extended multi-interference (MI) architecture testing in Env$_1$ (noise level $P=20\%$). As a proof of concept, we consider two interference types together, Gaussian noise and adversarial perturbation. In this setting every observation (state) can possibly undergo an interference with either Gaussian noise or Adversarial perturbation. CIQ-MI is capable of making correct action to solve (over 195.0) the testing environment when training with mixed interference types.}
\begin{adjustbox}{width=0.47\textwidth}
\begin{tabular}{|l|c|c|c|}
\hline
Train  \textbf{/ Test} & \textbf{Gaussian} & \textbf{Adversarial} & \textbf{Gaussian + Adversarial} \\ \hline
Gaussian & \textbf{195.1} & 154.2 & 96.3 \\ \hline
Adversarial & 153.9 & \textbf{195.0} & 105.1 \\ \hline
Gaussian + Adversarial & \textbf{195.0} & \textbf{195.0} & \textbf{195.0} \\ \hline
\end{tabular}
\end{adjustbox}
\label{tab:CIQ:MI}
\end{table}

\section{Conclusion}
Our experiments suggest that, although some DQN-based DRL algorithms can achieve high scores under the normal condition, their performance can be severely degraded in the presence of interference. 
In order to be resilient against interference, we propose CIQ, a novel causal-inference-driven DRL algorithm. Evaluated on a wide range of environments and multiple types of interferences, the CIQ results show consistently superior performance over several RL baseline methods. We investigate the improved resilience of CIQ by CLEVER-Q and AC-Rate metrics. Our demo code is available at \small{\texttt{github.com/huckiyang/Obs-Causal-Q-Network}}. 
 
\clearpage
{\small
\bibliography{iclr2021_conference}

\begin{thebibliography}{84}
\providecommand{\natexlab}[1]{#1}

\bibitem[{Alshiekh et~al.(2018)Alshiekh, Bloem, Ehlers, K{\"o}nighofer, Niekum,
  and Topcu}]{alshiekh2018safe}
Alshiekh, M.; Bloem, R.; Ehlers, R.; K{\"o}nighofer, B.; Niekum, S.; and Topcu,
  U. 2018.
\newblock Safe reinforcement learning via shielding.
\newblock In \emph{Thirty-Second AAAI Conference on Artificial Intelligence}.

\bibitem[{Ammanabrolu and Riedl(2019)}]{ammanabrolu2019transfer}
Ammanabrolu, P.; and Riedl, M. 2019.
\newblock Transfer in Deep Reinforcement Learning Using Knowledge Graphs.
\newblock In \emph{Proceedings of the Thirteenth Workshop on Graph-Based
  Methods for Natural Language Processing (TextGraphs-13)}, 1--10.

\bibitem[{Bareinboim, Forney, and Pearl(2015)}]{bareinboim2015bandits}
Bareinboim, E.; Forney, A.; and Pearl, J. 2015.
\newblock Bandits with unobserved confounders: A causal approach.
\newblock In \emph{Advances in Neural Information Processing Systems},
  1342--1350.

\bibitem[{Bengio(2013)}]{bengio2013deep}
Bengio, Y. 2013.
\newblock Deep learning of representations: Looking forward.
\newblock In \emph{International Conference on Statistical Language and Speech
  Processing}, 1--37. Springer.

\bibitem[{Bennett et~al.(2021)Bennett, Kallus, Li, and
  Mousavi}]{bennett2021off}
Bennett, A.; Kallus, N.; Li, L.; and Mousavi, A. 2021.
\newblock Off-policy evaluation in infinite-horizon reinforcement learning with
  latent confounders.
\newblock In \emph{International Conference on Artificial Intelligence and
  Statistics}, 1999--2007. PMLR.

\bibitem[{Brockman et~al.(2016)Brockman, Cheung, Pettersson, Schneider,
  Schulman, Tang, and Zaremba}]{brockman2016openai}
Brockman, G.; Cheung, V.; Pettersson, L.; Schneider, J.; Schulman, J.; Tang,
  J.; and Zaremba, W. 2016.
\newblock Openai gym.
\newblock \emph{arXiv preprint arXiv:1606.01540}.

\bibitem[{Dabney et~al.(2018)Dabney, Rowland, Bellemare, and
  Munos}]{dabney2018distributional}
Dabney, W.; Rowland, M.; Bellemare, M.~G.; and Munos, R. 2018.
\newblock Distributional reinforcement learning with quantile regression.
\newblock In \emph{Thirty-Second AAAI Conference on Artificial Intelligence}.

\bibitem[{Dhariwal et~al.(2017)Dhariwal, Hesse, Klimov, Nichol, Plappert,
  Radford, Schulman, Sidor, Wu, and Zhokhov}]{baselines}
Dhariwal, P.; Hesse, C.; Klimov, O.; Nichol, A.; Plappert, M.; Radford, A.;
  Schulman, J.; Sidor, S.; Wu, Y.; and Zhokhov, P. 2017.
\newblock OpenAI Baselines.
\newblock \url{https://github.com/openai/baselines}.

\bibitem[{Dulac-Arnold et~al.(2012)Dulac-Arnold, Denoyer, Preux, and
  Gallinari}]{dulac2012fast}
Dulac-Arnold, G.; Denoyer, L.; Preux, P.; and Gallinari, P. 2012.
\newblock Fast reinforcement learning with large action sets using
  error-correcting output codes for mdp factorization.
\newblock In \emph{Joint European Conference on Machine Learning and Knowledge
  Discovery in Databases}, 180--194. Springer.

\bibitem[{Everett(2021)}]{everett2021neural}
Everett, M. 2021.
\newblock Neural Network Verification in Control.
\newblock \emph{arXiv preprint arXiv:2110.01388}.

\bibitem[{Forney, Pearl, and Bareinboim(2017)}]{forney2017counterfactual}
Forney, A.; Pearl, J.; and Bareinboim, E. 2017.
\newblock Counterfactual data-fusion for online reinforcement learners.
\newblock In \emph{International Conference on Machine Learning}, 1156--1164.

\bibitem[{Fortunato et~al.(2018)Fortunato, Azar, Piot, Menick, Osband, Graves,
  Mnih, Munos, Hassabis, Pietquin et~al.}]{fortunato2017noisy}
Fortunato, M.; Azar, M.~G.; Piot, B.; Menick, J.; Osband, I.; Graves, A.; Mnih,
  V.; Munos, R.; Hassabis, D.; Pietquin, O.; et~al. 2018.
\newblock Noisy networks for exploration.
\newblock \emph{ICLR 2018, arXiv preprint arXiv:1706.10295}.

\bibitem[{Fox, Pakman, and Tishby(2015)}]{fox2015taming}
Fox, R.; Pakman, A.; and Tishby, N. 2015.
\newblock Taming the noise in reinforcement learning via soft updates.
\newblock \emph{arXiv preprint arXiv:1512.08562}.

\bibitem[{Franklin et~al.(1998)Franklin, Powell, Workman
  et~al.}]{franklin1998digital}
Franklin, G.~F.; Powell, J.~D.; Workman, M.~L.; et~al. 1998.
\newblock \emph{Digital control of dynamic systems}, volume~3.
\newblock Addison-wesley Menlo Park, CA.

\bibitem[{Goodfellow, Shlens, and Szegedy(2015)}]{goodfellow2014explaining}
Goodfellow, I.~J.; Shlens, J.; and Szegedy, C. 2015.
\newblock Explaining and harnessing adversarial examples.
\newblock \emph{ICLR}.

\bibitem[{Greenland, Pearl, and Robins(1999)}]{greenland1999causal}
Greenland, S.; Pearl, J.; and Robins, J.~M. 1999.
\newblock Causal diagrams for epidemiologic research.
\newblock \emph{Epidemiology}, 37--48.

\bibitem[{Gregor et~al.(2018)Gregor, Papamakarios, Besse, Buesing, and
  Weber}]{gregor2018temporal}
Gregor, K.; Papamakarios, G.; Besse, F.; Buesing, L.; and Weber, T. 2018.
\newblock Temporal difference variational auto-encoder.
\newblock \emph{arXiv preprint arXiv:1806.03107}.

\bibitem[{Greydanus et~al.(2018)Greydanus, Koul, Dodge, and
  Fern}]{greydanus2018visualizing}
Greydanus, S.; Koul, A.; Dodge, J.; and Fern, A. 2018.
\newblock Visualizing and Understanding Atari Agents.
\newblock In \emph{International Conference on Machine Learning}, 1792--1801.

\bibitem[{Grigorescu et~al.(2020)Grigorescu, Trasnea, Cocias, and
  Macesanu}]{grigorescu2020survey}
Grigorescu, S.; Trasnea, B.; Cocias, T.; and Macesanu, G. 2020.
\newblock A survey of deep learning techniques for autonomous driving.
\newblock \emph{Journal of Field Robotics}, 37(3): 362--386.

\bibitem[{Gu et~al.(2017)Gu, Holly, Lillicrap, and Levine}]{gu2017deep}
Gu, S.; Holly, E.; Lillicrap, T.; and Levine, S. 2017.
\newblock Deep reinforcement learning for robotic manipulation with
  asynchronous off-policy updates.
\newblock In \emph{2017 IEEE international conference on robotics and
  automation (ICRA)}, 3389--3396. IEEE.

\bibitem[{Ha and Schmidhuber(2018)}]{ha2018world}
Ha, D.; and Schmidhuber, J. 2018.
\newblock World models.
\newblock \emph{arXiv preprint arXiv:1803.10122}.

\bibitem[{Hafner et~al.(2018)Hafner, Lillicrap, Fischer, Villegas, Ha, Lee, and
  Davidson}]{hafner2018learning}
Hafner, D.; Lillicrap, T.; Fischer, I.; Villegas, R.; Ha, D.; Lee, H.; and
  Davidson, J. 2018.
\newblock Learning latent dynamics for planning from pixels.
\newblock \emph{arXiv preprint arXiv:1811.04551}.

\bibitem[{Harutyunyan et~al.(2016)Harutyunyan, Bellemare, Stepleton, and
  Munos}]{harutyunyan2016q}
Harutyunyan, A.; Bellemare, M.~G.; Stepleton, T.; and Munos, R. 2016.
\newblock Q lamda with Off-Policy Corrections.
\newblock In \emph{International Conference on Algorithmic Learning Theory},
  305--320. Springer.

\bibitem[{Helwegen, Louizos, and Forr{\'e}(2020)}]{helwegen2020improving}
Helwegen, R.; Louizos, C.; and Forr{\'e}, P. 2020.
\newblock Improving Fair Predictions Using Variational Inference In Causal
  Models.
\newblock \emph{arXiv preprint arXiv:2008.10880}.

\bibitem[{Higgins et~al.(2018)Higgins, Amos, Pfau, Racaniere, Matthey, Rezende,
  and Lerchner}]{higgins2018towards}
Higgins, I.; Amos, D.; Pfau, D.; Racaniere, S.; Matthey, L.; Rezende, D.; and
  Lerchner, A. 2018.
\newblock Towards a definition of disentangled representations.
\newblock \emph{arXiv preprint arXiv:1812.02230}.

\bibitem[{Huang et~al.(2017)Huang, Papernot, Goodfellow, Duan, and
  Abbeel}]{huang2017adversarial}
Huang, S.; Papernot, N.; Goodfellow, I.; Duan, Y.; and Abbeel, P. 2017.
\newblock Adversarial attacks on neural network policies.
\newblock \emph{arXiv preprint arXiv:1702.02284}.

\bibitem[{Igl et~al.(2018)Igl, Zintgraf, Le, Wood, and Whiteson}]{igl2018deep}
Igl, M.; Zintgraf, L.; Le, T.~A.; Wood, F.; and Whiteson, S. 2018.
\newblock Deep Variational Reinforcement Learning for POMDPs.
\newblock In \emph{International Conference on Machine Learning}, 2117--2126.

\bibitem[{Imbens and Rubin(2010)}]{imbens2010rubin}
Imbens, G.~W.; and Rubin, D.~B. 2010.
\newblock Rubin causal model.
\newblock In \emph{Microeconometrics}, 229--241. Springer.

\bibitem[{Jaber, Zhang, and Bareinboim(2019)}]{jaber2019causal}
Jaber, A.; Zhang, J.; and Bareinboim, E. 2019.
\newblock Causal identification under markov equivalence: Completeness results.
\newblock In \emph{International Conference on Machine Learning}, 2981--2989.

\bibitem[{Jaderberg et~al.(2017)Jaderberg, Mnih, Czarnecki, Schaul, Leibo,
  Silver, and Kavukcuoglu}]{jaderberg2017reinforcement}
Jaderberg, M.; Mnih, V.; Czarnecki, W.~M.; Schaul, T.; Leibo, J.~Z.; Silver,
  D.; and Kavukcuoglu, K. 2017.
\newblock Reinforcement learning with unsupervised auxiliary tasks.
\newblock \emph{ICLR}.

\bibitem[{Johansen et~al.(2015)Johansen, Cristofaro, S{\o}rensen, Hansen, and
  Fossen}]{johansen2015estimation}
Johansen, T.~A.; Cristofaro, A.; S{\o}rensen, K.; Hansen, J.~M.; and Fossen,
  T.~I. 2015.
\newblock On estimation of wind velocity, angle-of-attack and sideslip angle of
  small UAVs using standard sensors.
\newblock In \emph{2015 International Conference on Unmanned Aircraft Systems
  (ICUAS)}, 510--519. IEEE.

\bibitem[{Juliani et~al.(2018)Juliani, Berges, Vckay, Gao, Henry, Mattar, and
  Lange}]{juliani2018unity}
Juliani, A.; Berges, V.-P.; Vckay, E.; Gao, Y.; Henry, H.; Mattar, M.; and
  Lange, D. 2018.
\newblock Unity: A general platform for intelligent agents.
\newblock \emph{arXiv preprint arXiv:1809.02627}.

\bibitem[{Jung, Tian, and Bareinboim(2021)}]{jung2021estimating}
Jung, Y.; Tian, J.; and Bareinboim, E. 2021.
\newblock Estimating identifiable causal effects through double machine
  learning.
\newblock In \emph{Proceedings of the 35th AAAI Conference on Artificial
  Intelligence}.

\bibitem[{Kaelbling, Littman, and Cassandra(1998)}]{kaelbling1998planning}
Kaelbling, L.~P.; Littman, M.~L.; and Cassandra, A.~R. 1998.
\newblock Planning and acting in partially observable stochastic domains.
\newblock \emph{Artificial intelligence}, 101(1-2): 99--134.

\bibitem[{Kalashnikov et~al.(2018)Kalashnikov, Irpan, Pastor, Ibarz, Herzog,
  Jang, Quillen, Holly, Kalakrishnan, Vanhoucke et~al.}]{kalashnikov2018qt}
Kalashnikov, D.; Irpan, A.; Pastor, P.; Ibarz, J.; Herzog, A.; Jang, E.;
  Quillen, D.; Holly, E.; Kalakrishnan, M.; Vanhoucke, V.; et~al. 2018.
\newblock Qt-opt: Scalable deep reinforcement learning for vision-based robotic
  manipulation.
\newblock \emph{arXiv preprint arXiv:1806.10293}.

\bibitem[{Khemakhem et~al.(2021)Khemakhem, Monti, Leech, and
  Hyvarinen}]{khemakhem2021causal}
Khemakhem, I.; Monti, R.; Leech, R.; and Hyvarinen, A. 2021.
\newblock Causal autoregressive flows.
\newblock In \emph{International Conference on Artificial Intelligence and
  Statistics}, 3520--3528. PMLR.

\bibitem[{Killian, Ghassemi, and Joshi(2020)}]{killian2020counterfactually}
Killian, T.~W.; Ghassemi, M.; and Joshi, S. 2020.
\newblock Counterfactually Guided Policy Transfer in Clinical Settings.
\newblock \emph{arXiv preprint arXiv:2006.11654}.

\bibitem[{Kingma and Ba(2014)}]{kingma2014adam}
Kingma, D.~P.; and Ba, J. 2014.
\newblock Adam: A method for stochastic optimization.
\newblock \emph{arXiv preprint arXiv:1412.6980}.

\bibitem[{Kingma and Welling(2013)}]{kingma2013auto}
Kingma, D.~P.; and Welling, M. 2013.
\newblock Auto-encoding variational bayes.
\newblock \emph{arXiv preprint arXiv:1312.6114}.

\bibitem[{Lee et~al.(2018)Lee, Hou, Mandalika, Lee, Choudhury, and
  Srinivasa}]{lee2018bayesian}
Lee, G.; Hou, B.; Mandalika, A.; Lee, J.; Choudhury, S.; and Srinivasa, S.~S.
  2018.
\newblock Bayesian policy optimization for model uncertainty.
\newblock \emph{arXiv preprint arXiv:1810.01014}.

\bibitem[{Lin et~al.(2017)Lin, Hong, Liao, Shih, Liu, and Sun}]{lin2017tactics}
Lin, Y.-C.; Hong, Z.-W.; Liao, Y.-H.; Shih, M.-L.; Liu, M.-Y.; and Sun, M.
  2017.
\newblock Tactics of adversarial attack on deep reinforcement learning agents.
\newblock \emph{arXiv preprint arXiv:1703.06748}.

\bibitem[{Louizos et~al.(2017)Louizos, Shalit, Mooij, Sontag, Zemel, and
  Welling}]{louizos2017causal}
Louizos, C.; Shalit, U.; Mooij, J.~M.; Sontag, D.; Zemel, R.; and Welling, M.
  2017.
\newblock Causal effect inference with deep latent-variable models.
\newblock In \emph{Advances in Neural Information Processing Systems},
  6446--6456.

\bibitem[{Lu, Sch{\"o}lkopf, and
  Hern{\'a}ndez-Lobato(2018)}]{lu2018deconfounding}
Lu, C.; Sch{\"o}lkopf, B.; and Hern{\'a}ndez-Lobato, J.~M. 2018.
\newblock Deconfounding reinforcement learning in observational settings.
\newblock \emph{arXiv preprint arXiv:1812.10576}.

\bibitem[{Lynch et~al.(2020)Lynch, Khansari, Xiao, Kumar, Tompson, Levine, and
  Sermanet}]{lynch2020learning}
Lynch, C.; Khansari, M.; Xiao, T.; Kumar, V.; Tompson, J.; Levine, S.; and
  Sermanet, P. 2020.
\newblock Learning latent plans from play.
\newblock In \emph{Conference on Robot Learning}, 1113--1132. PMLR.

\bibitem[{Madumal et~al.(2020)Madumal, Miller, Sonenberg, and
  Vetere}]{madumal2020explainable}
Madumal, P.; Miller, T.; Sonenberg, L.; and Vetere, F. 2020.
\newblock Explainable reinforcement learning through a causal lens.
\newblock In \emph{Proceedings of the AAAI Conference on Artificial
  Intelligence}, volume~34, 2493--2500.

\bibitem[{Mirowski et~al.(2016)Mirowski, Pascanu, Viola, Soyer, Ballard,
  Banino, Denil, Goroshin, Sifre, Kavukcuoglu et~al.}]{mirowski2016learning}
Mirowski, P.; Pascanu, R.; Viola, F.; Soyer, H.; Ballard, A.~J.; Banino, A.;
  Denil, M.; Goroshin, R.; Sifre, L.; Kavukcuoglu, K.; et~al. 2016.
\newblock Learning to navigate in complex environments.
\newblock \emph{arXiv preprint arXiv:1611.03673}.

\bibitem[{Mnih et~al.(2016)Mnih, Badia, Mirza, Graves, Lillicrap, Harley,
  Silver, and Kavukcuoglu}]{mnih2016asynchronous}
Mnih, V.; Badia, A.~P.; Mirza, M.; Graves, A.; Lillicrap, T.; Harley, T.;
  Silver, D.; and Kavukcuoglu, K. 2016.
\newblock Asynchronous methods for deep reinforcement learning.
\newblock In \emph{International conference on machine learning}, 1928--1937.

\bibitem[{Mnih et~al.(2015)Mnih, Kavukcuoglu, Silver, Rusu, Veness, Bellemare,
  Graves, Riedmiller, Fidjeland, Ostrovski et~al.}]{mnih2015human}
Mnih, V.; Kavukcuoglu, K.; Silver, D.; Rusu, A.~A.; Veness, J.; Bellemare,
  M.~G.; Graves, A.; Riedmiller, M.; Fidjeland, A.~K.; Ostrovski, G.; et~al.
  2015.
\newblock Human-level control through deep reinforcement learning.
\newblock \emph{Nature}, 518(7540): 529.

\bibitem[{Moreno et~al.(2018)Moreno, Humplik, Papamakarios, Pires, Buesing,
  Heess, and Weber}]{moreno2018neural}
Moreno, P.; Humplik, J.; Papamakarios, G.; Pires, B.~A.; Buesing, L.; Heess,
  N.; and Weber, T. 2018.
\newblock Neural belief states for partially observed domains.
\newblock In \emph{NeurIPS 2018 workshop on Reinforcement Learning under
  Partial Observability}.

\bibitem[{Nagabandi et~al.(2018)Nagabandi, Clavera, Liu, Fearing, Abbeel,
  Levine, and Finn}]{nagabandi2018learning}
Nagabandi, A.; Clavera, I.; Liu, S.; Fearing, R.~S.; Abbeel, P.; Levine, S.;
  and Finn, C. 2018.
\newblock Learning to adapt in dynamic, real-world environments through
  meta-reinforcement learning.
\newblock \emph{arXiv preprint arXiv:1803.11347}.

\bibitem[{Osband et~al.(2019)Osband, Doron, Hessel, Aslanides, Sezener,
  Saraiva, McKinney, Lattimore, Szepezvari, Singh et~al.}]{osband2019behaviour}
Osband, I.; Doron, Y.; Hessel, M.; Aslanides, J.; Sezener, E.; Saraiva, A.;
  McKinney, K.; Lattimore, T.; Szepezvari, C.; Singh, S.; et~al. 2019.
\newblock Behaviour suite for reinforcement learning.
\newblock \emph{arXiv preprint arXiv:1908.03568}.

\bibitem[{Papadimitriou and Tsitsiklis(1987)}]{papadimitriou1987complexity}
Papadimitriou, C.~H.; and Tsitsiklis, J.~N. 1987.
\newblock The complexity of Markov decision processes.
\newblock \emph{Mathematics of operations research}, 12(3): 441--450.

\bibitem[{Pearl(1995{\natexlab{a}})}]{pearl1995causal}
Pearl, J. 1995{\natexlab{a}}.
\newblock Causal diagrams for empirical research.
\newblock \emph{Biometrika}, 82(4): 669--688.

\bibitem[{Pearl(1995{\natexlab{b}})}]{pearl1995testability}
Pearl, J. 1995{\natexlab{b}}.
\newblock On the testability of causal models with latent and instrumental
  variables.
\newblock In \emph{Proceedings of the Eleventh conference on Uncertainty in
  artificial intelligence}, 435--443. Morgan Kaufmann Publishers Inc.

\bibitem[{Pearl(2009)}]{pearl2009causality}
Pearl, J. 2009.
\newblock \emph{Causality}.
\newblock Cambridge university press.

\bibitem[{Pearl(2019)}]{pearl2019seven}
Pearl, J. 2019.
\newblock The seven tools of causal inference, with reflections on machine
  learning.
\newblock \emph{Communications of the ACM}, 62(3): 54--60.

\bibitem[{Pearl, Glymour, and Jewell(2016)}]{pearl2016causal}
Pearl, J.; Glymour, M.; and Jewell, N.~P. 2016.
\newblock \emph{Causal inference in statistics: A primer}.
\newblock John Wiley \& Sons.

\bibitem[{Raghunathan et~al.(2019)Raghunathan, Xie, Yang, Duchi, and
  Liang}]{raghunathan2019adversarial}
Raghunathan, A.; Xie, S.~M.; Yang, F.; Duchi, J.~C.; and Liang, P. 2019.
\newblock Adversarial training can hurt generalization.
\newblock \emph{arXiv preprint arXiv:1906.06032}.

\bibitem[{Robins, Rotnitzky, and Zhao(1995)}]{robins1995analysis}
Robins, J.~M.; Rotnitzky, A.; and Zhao, L.~P. 1995.
\newblock Analysis of semiparametric regression models for repeated outcomes in
  the presence of missing data.
\newblock \emph{Journal of the american statistical association}, 90(429):
  106--121.

\bibitem[{Rothman and Greenland(2005)}]{rothman2005causation}
Rothman, K.~J.; and Greenland, S. 2005.
\newblock Causation and causal inference in epidemiology.
\newblock \emph{American journal of public health}, 95(S1): S144--S150.

\bibitem[{Rubin(1974)}]{rubin1974estimating}
Rubin, D.~B. 1974.
\newblock Estimating causal effects of treatments in randomized and
  nonrandomized studies.
\newblock \emph{Journal of educational Psychology}, 66(5): 688.

\bibitem[{Saunders et~al.(2018)Saunders, Sastry, Stuhlmueller, and
  Evans}]{saunders2018trial}
Saunders, W.; Sastry, G.; Stuhlmueller, A.; and Evans, O. 2018.
\newblock Trial without error: Towards safe reinforcement learning via human
  intervention.
\newblock In \emph{Proceedings of the 17th International Conference on
  Autonomous Agents and MultiAgent Systems}, 2067--2069. International
  Foundation for Autonomous Agents and Multiagent Systems.

\bibitem[{Schmidhuber(1991)}]{schmidhuber1991reinforcement}
Schmidhuber, J. 1991.
\newblock Reinforcement learning in Markovian and non-Markovian environments.
\newblock In \emph{Advances in neural information processing systems},
  500--506.

\bibitem[{Schmidhuber(1992)}]{schmidhuber1992learning}
Schmidhuber, J. 1992.
\newblock Learning complex, extended sequences using the principle of history
  compression.
\newblock \emph{Neural Computation}, 4(2): 234--242.

\bibitem[{Shalit, Johansson, and Sontag(2017)}]{shalit2017estimating}
Shalit, U.; Johansson, F.~D.; and Sontag, D. 2017.
\newblock Estimating individual treatment effect: generalization bounds and
  algorithms.
\newblock In \emph{Proceedings of the 34th International Conference on Machine
  Learning-Volume 70}, 3076--3085. JMLR. org.

\bibitem[{Sharma, Kiciman et~al.(2019)}]{sharma2019dowhy}
Sharma, A.; Kiciman, E.; et~al. 2019.
\newblock Do{W}hy A Python package for causal inference.
\newblock \emph{KDD 2019 workshop}.

\bibitem[{Silver et~al.(2017)Silver, Schrittwieser, Simonyan, Antonoglou,
  Huang, Guez, Hubert, Baker, Lai, Bolton et~al.}]{silver2017mastering}
Silver, D.; Schrittwieser, J.; Simonyan, K.; Antonoglou, I.; Huang, A.; Guez,
  A.; Hubert, T.; Baker, L.; Lai, M.; Bolton, A.; et~al. 2017.
\newblock Mastering the game of Go without human knowledge.
\newblock \emph{Nature}, 550(7676): 354.

\bibitem[{Su et~al.(2018)Su, Zhang, Chen, Yi, Chen, and Gao}]{su2018robustness}
Su, D.; Zhang, H.; Chen, H.; Yi, J.; Chen, P.-Y.; and Gao, Y. 2018.
\newblock Is robustness the cost of accuracy?--a comprehensive study on the
  robustness of 18 deep image classification models.
\newblock In \emph{ECCV}, 631--648.

\bibitem[{Sutton et~al.(1998)Sutton, Barto, Bach
  et~al.}]{sutton1998reinforcement}
Sutton, R.~S.; Barto, A.~G.; Bach, F.; et~al. 1998.
\newblock \emph{Reinforcement learning: An introduction}.
\newblock MIT press.

\bibitem[{Szegedy et~al.(2014)Szegedy, Zaremba, Sutskever, Bruna, Erhan,
  Goodfellow, and Fergus}]{szegedy2013intriguing}
Szegedy, C.; Zaremba, W.; Sutskever, I.; Bruna, J.; Erhan, D.; Goodfellow, I.;
  and Fergus, R. 2014.
\newblock Intriguing properties of neural networks.
\newblock \emph{International Conference on Learning Representations}.

\bibitem[{Tai, Paolo, and Liu(2017)}]{tai2017virtual}
Tai, L.; Paolo, G.; and Liu, M. 2017.
\newblock Virtual-to-real deep reinforcement learning: Continuous control of
  mobile robots for mapless navigation.
\newblock In \emph{2017 IEEE/RSJ International Conference on Intelligent Robots
  and Systems (IROS)}, 31--36. IEEE.

\bibitem[{Tennenholtz, Mannor, and Shalit(2019)}]{tennenholtz2019off}
Tennenholtz, G.; Mannor, S.; and Shalit, U. 2019.
\newblock Off-Policy Evaluation in Partially Observable Environments.
\newblock \emph{arXiv preprint arXiv:1909.03739}.

\bibitem[{Van~Hasselt, Guez, and Silver(2016)}]{van2016deep}
Van~Hasselt, H.; Guez, A.; and Silver, D. 2016.
\newblock Deep reinforcement learning with double q-learning.
\newblock In \emph{Thirtieth AAAI conference on artificial intelligence}.

\bibitem[{Weng et~al.(2018)Weng, Zhang, Chen, Yi, Su, Gao, Hsieh, and
  Daniel}]{weng2018evaluating}
Weng, T.-W.; Zhang, H.; Chen, P.-Y.; Yi, J.; Su, D.; Gao, Y.; Hsieh, C.-J.; and
  Daniel, L. 2018.
\newblock Evaluating the robustness of neural networks: An extreme value theory
  approach.
\newblock \emph{arXiv preprint arXiv:1801.10578}.

\bibitem[{Yan, Xu, and Liu(2016)}]{yan2016can}
Yan, C.; Xu, W.; and Liu, J. 2016.
\newblock Can you trust autonomous vehicles: Contactless attacks against
  sensors of self-driving vehicle.
\newblock \emph{DEFCON24}.

\bibitem[{Yang et~al.(2020{\natexlab{a}})Yang, Liu, Gandhe, Gu, Raju,
  Filimonov, and Bulyko}]{yang2020multi}
Yang, C.-H.~H.; Liu, L.; Gandhe, A.; Gu, Y.; Raju, A.; Filimonov, D.; and
  Bulyko, I. 2020{\natexlab{a}}.
\newblock Multi-task Language Modeling for Improving Speech Recognition of Rare
  Words.
\newblock \emph{arXiv preprint arXiv:2011.11715}.

\bibitem[{Yang et~al.(2019)Yang, Liu, Chen, Ma, and Tsai}]{yang2019causal}
Yang, C.-H.~H.; Liu, Y.-C.; Chen, P.-Y.; Ma, X.; and Tsai, Y.-C.~J. 2019.
\newblock When causal intervention meets adversarial examples and image masking
  for deep neural networks.
\newblock In \emph{2019 IEEE International Conference on Image Processing
  (ICIP)}, 3811--3815. IEEE.

\bibitem[{Yang et~al.(2020{\natexlab{b}})Yang, Qi, Chen, Ouyang, Hung, Lee, and
  Ma}]{yang2020enhanced}
Yang, C.-H.~H.; Qi, J.; Chen, P.-Y.; Ouyang, Y.; Hung, I.-T.~D.; Lee, C.-H.;
  and Ma, X. 2020{\natexlab{b}}.
\newblock Enhanced Adversarial Strategically-Timed Attacks Against Deep
  Reinforcement Learning.
\newblock In \emph{ICASSP 2020-2020 IEEE International Conference on Acoustics,
  Speech and Signal Processing (ICASSP)}, 3407--3411. IEEE.

\bibitem[{Yurtsever et~al.(2020)Yurtsever, Lambert, Carballo, and
  Takeda}]{yurtsever2020survey}
Yurtsever, E.; Lambert, J.; Carballo, A.; and Takeda, K. 2020.
\newblock A survey of autonomous driving: Common practices and emerging
  technologies.
\newblock \emph{IEEE Access}, 8: 58443--58469.

\bibitem[{Zhang et~al.(2020)Zhang, Lyle, Sodhani, Filos, Kwiatkowska, Pineau,
  Gal, and Precup}]{zhang2020invariant}
Zhang, A.; Lyle, C.; Sodhani, S.; Filos, A.; Kwiatkowska, M.; Pineau, J.; Gal,
  Y.; and Precup, D. 2020.
\newblock Invariant causal prediction for block mdps.
\newblock In \emph{International Conference on Machine Learning}, 11214--11224.
  PMLR.

\bibitem[{Zhang, Zhang, and Li(2020)}]{zhang2020causala}
Zhang, C.; Zhang, K.; and Li, Y. 2020.
\newblock A Causal View on Robustness of Neural Networks.
\newblock \emph{Advances in Neural Information Processing Systems}, 33.

\bibitem[{Zhang and Bareinboim(2020)}]{zhang2020designing}
Zhang, J.; and Bareinboim, E. 2020.
\newblock Designing optimal dynamic treatment regimes: A causal reinforcement
  learning approach.
\newblock In \emph{International Conference on Machine Learning}, 11012--11022.
  PMLR.

\bibitem[{Zhang and Bareinboim(2021)}]{zhang2021bounding}
Zhang, J.; and Bareinboim, E. 2021.
\newblock Bounding Causal Effects on Continuous Outcome.
\newblock In \emph{Proceedings of the 35nd AAAI Conference on Artificial
  Intelligence}.

\bibitem[{Zhang, Kumor, and Bareinboim(2020)}]{zhang2020causalb}
Zhang, J.; Kumor, D.; and Bareinboim, E. 2020.
\newblock Causal imitation learning with unobserved confounders.
\newblock \emph{Advances in Neural Information Processing Systems}, 33.

\end{thebibliography}
}
\appendix
\onecolumn
\section{Appendix}
Our supplementary sections included: 
\begin{itemize}
\item \textbf{A.} Proof of the CLEVER-Q Theorem and Additional Robustness Measurements
\item \textbf{B.} Implementation Details and Additional Results
\item \textbf{C.} Causal Relation Evaluation and Average Treatment Effects in CIQ 
Networks
\item \textbf{D.} Ablation Studies
\end{itemize}

\section{A. Proof of the CLEVER-Q Theorem and Additional Robustness Measurements}

\subsection{Proof of the CLEVER-Q Theorem}
\label{sup:b:proof}

Here we provide a comprehensive score (CLEVER-Q) for evaluating the robustness of a Q-network model by extending the CLEVER robustness score ~\citep{weng2018evaluating} designed for classification tasks to Q-network based DRL tasks. Consider an $\ell_p$-norm bounded ($p \geq 1$) perturbation $\delta$ to the state $s_t$. We first derive a lower bound $\beta_L$ on the minimal perturbation to $s_t$ for altering the action with the top Q-value, i.e., the greedy action. For a given $s_t$ and a Q-network, this lower bound $\beta_L$ provides a robustness guarantee that the greedy action at $s_t$ will be the same as that of \textit{any} perturbed state $s_t+\delta$, as long as the perturbation level $\|\delta\|_p \leq \beta_L$. Therefore, the larger the value $\beta_L$ is, the more resilience of the Q-network against perturbations can be guaranteed. Our CLEVER-Q score uses the extreme value theory to evaluate the lower bound  $\beta_L$ as a robustness metric for benchmarking different Q-network models.

\begin{theorem}

Consider a Q-network $Q(s,a)$ and a state $s_t$. Let $\mathcal{A}^* = \argmax_{a} Q(s_t, a)$ be the set of greedy (best) actions having the highest Q-value at $s_t$ according to the Q-network. Define $g_a(s_t) = Q(s_t, \mathcal{A}^*) - Q(s_t, a)$ for 
every action $a$, where $Q(s_t, \mathcal{A}^*)$ denotes the best Q-value at $s_t$.
Assume $g_a(s_t)$ is locally Lipschitz continuous\footnote{Here locally Lipschitz continuous means $g_a(s_t)$ is Lipschitz continuous within the $\ell_p$ ball centered at $s_t$ with radius $R_p$. We follow the same definition as in \citep{weng2018evaluating}.} with its local Lipschitz constant denoted by $L_q^a$, where $1/p+1/q = 1$ and $p \geq 1$. Then for any $p\geq 1$, define the lower bound
\begin{align*}
   \beta_{L}= min_{a \notin  \mathcal{A}^*} g_a(s_t) / L_q^a.
\end{align*}
Then for any $\delta $ such that $\|\delta\|_p \leq \beta_L$, 
\begin{align*}
   \argmax_{a} Q(s_t, a) = \argmax_{a} Q(s_t + \delta, a)
\end{align*}
\end{theorem}
\begin{proof}
Because $g_a(s_t)$ is locally Lipschitz continuous, by Holder's inequality, we have  
\begin{equation}
\label{eqn_aa}
    |g_a(x)-g_a(y)| \leq  L_{q}^a ||x-y||_{p},
\end{equation}
for any $x,y$ within the $R_p$-ball centered at $s_t$.
Now let $x=s_t$ and $ y = s_t + \delta$, where $\delta$ is some perturbation. Then    
\begin{equation}
g_a(s_t) - L_{q}^a ||\delta||_{p}  \leq  g_a(s_t+\delta) \leq  g_a(s_t) + L_{q}^a. ||\delta||_{p}
\end{equation}
Note that if $g_a(s_t+\delta) \geq 0$, then $A^*$ still remains as the top Q-value action set at state $s_t+\delta$. Moreover, $g_a(s_t) - L_q^a ||\delta||_p \geq 0$ implies  $g_a(s_t+\delta) \geq 0$. Therefore,
\begin{equation}
    ||\delta||_{p} \leq g_a(s_t) / L_{q}^a,  
\label{eqn_LB_2}
\end{equation}
provides a robustness guarantee that ensures 
$Q(s_t+\delta, \mathcal{A}^*) \geq Q(s_t+\delta, a)$ for any $\delta$ satisfying Eq. \eqref{eqn_aa}. 
 Finally, to provide a robustness guarantee that $Q(s_t+\delta, \mathcal{A}^*) \geq Q(s_t+\delta, a)$ for any action $a \notin \mathcal{A}^*$, it suffices to take the minimum value of the bound (for each $a$) in  Eq. \eqref{eqn_aa} over all actions other than $a^*$, which gives the lower bound 
\begin{equation}
\beta_L = min_ {a \notin A^*} g_a(s_t) / L^a_q 
\end{equation} 
\label{eq:4}
\end{proof}

For computing $\beta_L$, while the numerator is easy to obtain, the local Lipschitz constant $L_q^a$ cannot be directly computed. In our implementation, by using the fact that $L_q^a$ is equivalent to the local maximum gradient norm (in $\ell_q$ norm),
we use the same sampling technique from extreme value theory as proposed in \citep{weng2018evaluating} for estimating $L_q^a$.

\begin{algorithm}
    \caption{CIQ Training}\label{CIA_training}
    \begin{algorithmic}[1]
    \STATE Inputs: $Agent$, $NoisyEnv$, $Oracle$, $max\_step$,
     $NoisyEnv\_test$, $target$, $eval\_steps$
    \STATE Initialize: $t = 0$, $score = 0$, $s'_t = NoisyEnv$.reset() 
    \WHILE {$t < max\_step$ \AND $score < target$} 
        \STATE $i_t$ = $oracle$($NoisyEnv$, $t$)
        \STATE $a_t$ = $Agent$.act($s_t'$, $i_t$)
        \STATE $s_{t+1}'$, $r_t$, $done$ = $NoisyEnv$.step($a_t$)
        \STATE $Agent$.learn($s_t'$, $a_t$, $r_t$, $s_{t+1}'$, $i_t$)
        \IF {$t \in eval\_steps$}
            \STATE $score$ = $Agent$.evaluate($NoisyEnv\_test$)
        \ENDIF
        \IF {$done$}
            \STATE $s_t'$ = $NoisyEnv$.reset()
        \ELSE
            \STATE $s_t'$ = $s_{t+1}'$
        \ENDIF
        
        \STATE $t$ = $t$ + 1
    \ENDWHILE
    \STATE Return $Agent$
    \end{algorithmic}
\end{algorithm}

\subsection{Background and Training Setting}
\label{sup:c:traing:setting}
To scale to high-dimensional problems, one can use a parameterized deep neural network $Q(s, a; \theta)$ to approximate the Q-function, and the network $Q(s, a; \theta)$ is referred to as the deep Q-network (DQN).
The DQN algorithm \citep{mnih2015human} updates parameter $\theta$ according to the loss function:
\begin{align*}
    L^{\text{DQN}}(\theta)= \mathbb E_{(s_t, a_t, r_t, s_{t+1}) \sim D}
    \Big[
    (y^{\text{DQN}}_t - Q(s_t, a_t; \theta))^2
    \Big]
\end{align*}
where the transitions $(s_t, a_t, r_t, s_{t+1})$ are uniformly sampled from the replay buffer $D$ of previously observed transitions, and
$
    y^{\text{DQN}}_t = r_t + \gamma \max_{a} Q(s_{t+1}, a; \theta^{-})
$
is the DQN target with $\theta^{-}$ being the target network parameter periodically updated by $\theta$. 

Double DQN (DDQN) \citep{van2016deep} further improves the performance by modifying the target to 
$
    y^{\text{DDQN}}_t = r_t + \gamma Q(s_{t+1}, \argmax_{a}Q(s_{t+1}, a; \theta) ; \theta^{-}).
$
Prioritized replay is another DQN improvement which samples transitions $(s_t, a_t, r_t, s_{t+1})$ from the replay buffer according to the probabilities $p_t$ proportional to their temporal difference (TD) error:
$
    p_t \propto |y^{\text{DDQN}}_t - Q(s_t, a_t; \theta) |^\alpha
$
where $\alpha$ is a hyperparameter.

We use Pytorch 1.2 to design both DQN and causal inference Q (CIQ) networks in our experiments. Our code can be found in the supplementary material.
We use Nvidia GeForce RTX 2080 Ti GPUs with CUDA 10.0 for our experiments. 
We use the Quantile Huber loss~\citep{dabney2018distributional} $\mathcal{L}_{\kappa}$ for DQN models with $\kappa = 1$ in Sup-Eq. \ref{eq:s:qh:2}, which allows less dramatic changes from Huber loss: 
\begin{equation} 
\mathcal{L}_{\kappa}(u)=\left\{\begin{array}{ll}
{\frac{1}{2} u^{2},} & {\text { if }|u| \leq \kappa} \\
{\kappa\left(|u|-\frac{1}{2} \kappa\right),} & {\text { otherwise }}
\end{array}\right.
\label{eq:s:1}
\end{equation}
The quantile Huber loss~\citep{dabney2018distributional} is the asymmetric variant of the Huber loss for quantile $\tau \in[0,1]$ from Sup-Eq. \ref{eq:s:1}:
\begin{equation}
\rho_{\tau}^{\kappa}(u)=\left|\tau-\delta_{\{u<0\}}\right| \mathcal{L}_{\kappa}(u).
\label{eq:s:qh:2}
\end{equation}

After the a maximum update step in the temporal loss $u$ in Sup-Eq.~\ref{eq:s:1}, we synchronize $\theta_{i}^{-}$ with $\theta_{i}$ follow the implementation from the OpenAI baseline~\citep{baselines} in Sup-Eq~\ref{eq:s:2}:
\begin{equation}
u_{i}\left(\theta_{i}\right)=\mathbb{E}_{ }\left(\underbrace{
y^{\text{DDQN}}
}_{\theta_{\text {target }}}-\underbrace{Q\left(s, a ; \theta_{i}\right)}_{\theta_{\text {local }}}\right)^{2}.
\label{eq:s:2}
\end{equation}
We use the soft-update~\citep{fox2015taming} to update the DQN target network as in Sup-Eq~\ref{eq:s:3}: 
\begin{equation}
    \theta_{\text{local}} = \tau \times \theta_{\text{local}} + (1 - \tau) \times\theta_{\text{target}},
\label{eq:s:3}
\end{equation}
where $\theta_{\text{target}}$ and $\theta_{\text{local}}$ represent the two neural networks in DQN and $\tau$ is the soft update parameter depending on the task.

For each environment, in additional to the 5 baselines described in Section 4.2, we also evaluate the performance of common DQN improvements such as deep double Q-networks (DDQN) for DDQN with dueling (DDQN$_{d}$), DDQN with a prioritized replay (DDQN$_{p}$), DDQN with a joint-training interference classifier (DDQN-CF), and DDQN with a safe action reply (DDQN-SA).
We test each model against four types of interference, Gaussian, Adversarial, Blackout, and Frozen Frame, with $p^I \in  [10\%, 20\%, 30\%, 40\%, 50\%]$. We also consider a non-stationary noise-level sampling from a cosine-wave in a range of [0\%, 30\%] for every ten steps. CIQ shows a better and continuous performance to solve the environments before the noise level attaining 40\% and under the cosine-noise. Compared to variational based DQNs methods, joint-trained DDQN-CF show a much obvious advantages when the noise levels are in the range of 40\% to 50\%.

\subsection{Env$_{1}$: Cartpole Environment.}
We use a four-layer neural network, which included an input layer, two 32-unit wide ReLU hidden layers, and an output layer (2 dimensions). The observation dimension of Cartpole-v1~\citep{brockman2016openai} is 4 and the input stacks 4 consecutive observations.
The dimension of the input layer is [$4\times4$]. We design a replay buffer with a memory of 100,000, with a mini-batch size of 32, the discount factor $\gamma$ is set to 0.99, the $\tau$ for a soft update of target parameter is $5\times10^{-3}$, a learning rate for Adam~\citep{kingma2014adam} optimization is $5\times10^{-4}$, a regularization term for weight decay is $1\times{10^{-4}}$, 
the coefficient $\alpha$ for importance sampling exponent is 0.6, the coefficient of prioritization exponent is 0.4.

\subsection{Env$_{2}$: 3D Banana Collector Environment.}
We utilize the Unity Machine Learning Agents Toolkit~\citep{juliani2018unity}, which is an open-source\footnote{Source: https://github.com/Unity-Technologies/ml-agents} and reproducible 3D rendering environment for the task of Banana Collector. We use  open-source graphic rendering version~\footnote{Source: https://github.com/udacity/deep-reinforcement-learning/tree/master/p1$\_$navigation} with Unity backbone~\citep{juliani2018unity} for reproducible DQN experiments, which is designed to render the collector agent for both Linux and Windows systems. A reward of $+1$ is provided for collecting a yellow banana, and a reward of $-1$ is provided for collecting a blue banana.
We use a six-layer deep network, which includes an input layer, three 64-unit fully-connected ReLU hidden layers, and an output layer (2 dimensions). We use [$37 \times 4$] for our input layer, which composes from the observation dimension (37) and the stacked input of 4 consecutive observations.
We design a replay buffer with a memory of 100,000, with a mini-batch size of 32, the discount factor $\gamma$ is equal to 0.99, the $\tau$ for a soft update of target parameter is $10^{-3}$,
a learning rate for Adam~\citep{kingma2014adam} optimization is $5\times10^{-4}$, a regularization term for weight decay is $1\times{10^{-4}}$, 
the coefficient $\alpha$ for importance sampling exponent is 0.6, the coefficient of prioritization exponent is 0.4.

\subsection{Env$_{3}$: Lunar Lander Environment.}
The lunar lander-v2~\citep{brockman2016openai} is one of the most challenging environments with discrete actions.
The observation dimension of Lunar Lander-v2~\citep{brockman2016openai} is 8 and the input stacks 10 consecutive observations. The objective of the game is to navigate the lunar lander spaceship to a targeted landing spot without a collision. A collection of six discrete actions controls two real-valued vectors ranging from -1 to +1. The first dimension controls the main engine on and off numerically, and the second dimension throttles from 50\% to 100\% power. The following two actions represent for firing left, and the last two actions represent for firing the right engine.  
The dimension of the input layer is [$8\times10$].
We design a 7-layers neural network for this task, which includes 1 input layer, 2 layer of 32 unit wide fully-connected ReLU network, 2 layers deep 64-unit wide ReLU networks, (for all DQNs), 1 layer of 16 unit wide fully-connected ReLU network, and 1 output layer (4 dimensions).
The replay buffer size is 500,000; the minimum batch size is 64, the discount factor is 0.99, the $\tau$ for a soft update of target parameters is $10^{-3}$, the learning rate is $5\times10^{-4}$, the minimal step for reset memory buffer is 50. 
We train each model 1,000 times for each case and report the mean of the average final performance (average over all types of interference). Env$_{3}$ is a challenging task owing to often receive negative reward during the training. We thus consider a non-stationary noise-level sampling from a cosine-wave in a narrow range of [0\%, 20\%] for every ten steps. Results suggest CIQ could still solve the environment before the noise-level going over to 30\%. For the various noisy test, CIQ attains a best performance over 200.0 the other DQNs algorithms (skipping the table since only CIQ and DQN-CF have solved the environment over 200.0 training with adversarial and blackout interference.)

\subsection{Env$_{4}$: Pixel Cartpole Environment}
To observe pixel inputs of Cartpole-v1 as states, we use a screen-wrapper with an original size of [400, 600, 3]. We first resize the original frame into a single gray-scale channel, [100, 150] from the RGN2GRAY function in the OpenCV. The implementation details are shown in the "pixel\_tool.py" and "cartpole\_pixel.py", which could be refereed to the submitted supplementary code. Then we stack 4 consecutive gray-scale frames as the input.
We design a 7-layer DQN model, which included input layer, the first hidden layer convolves 32 filters of a [$8\times8$] kernel with stride 4, the second hidden layer convolves 64 filters of a [$4\times4$] kernel with stride 2, the third layer is a fully-connected layer with 128 units, from fourth to fifth layers are fully-connected layer with 64 units, and the output layer (2 dimensions). The replay buffer size is 500,000; the minimum batch size is 32, the discount factor is 0.99, the $\tau$ for a soft update of target parameters is $10^{-3}$, the learning rate is $5\times10^{-4}$, the minimal step for reset memory buffer is 1000.

\subsection{Train and Test on Different Noise Level}
\label{sup:c:diff:noise}
We consider settings with different training and testing noise levels for CIQ evaluation. The (train, test)\% case trains with train\% noise then tests with test\% noise. We observe that CIQ have the capability of learning transformable q-value estimation, which attain a succeed score of 195.00 in the noise level 30 $\pm$ 10\%. Meanwhile, other DQNs methods included DDQN-CF, DVRLQ-CF, DDQN-SA perform a general performance decay in the test on different noise level. This result would be limited to the generalization of power and challenges~\citep{bengio2013deep, higgins2018towards} in as disentangle unseen state of a single parameterized deep network.

\subsection{Markovian Noise}
\label{sup:sec:markov}
We also provide a dynamic noise study for CIQ training with i.i.d. Gaussian interference, $p^{\mathcal I}=0.2$; testing with Markov distribution $P(i_t\!=\!1|i_{t-1}\!=\!1)\!=\! 0.55$, 
$P(i_t\!=\!1|i_{t-1}\!=\!0)\!=\!0.05$, stationary $p^{\mathcal I}=0.1$ testing in Env$_{1}$ and Env$_{2}$. This experiment shows the learning power against unseen Markovian interference in Table~\ref{tab:markov}, which further confirms CIQ's ability against unseen interference distribution and dynamics. 

\begin{table}[ht!]
\centering
\caption{CIQ training against unseen interference distribution and dynamics
}
\label{tab:markov}
\begin{tabular}{|c|c|c|c|c|c|c|c|}
\hline
Model       & DQN   & CIQ   & DQN-CF & DQN-SA & DVRLQ  & DQN-VAE  & DQN-CEVAE\\ \hline
Env$_1$ Markov & 112.3 & \textbf{195.0} & 181.4  & 131.4  & 112.1& 163.7& 155.6 \\ \hline
Env$_2$ Markov & 9.4   & \textbf{12.1}  & 11.7   & 9.1    & 11.5 & 11.2 & 11.2\\ \hline
\end{tabular}
\end{table}

\subsection{Advantages of Training with Interference Labels}

 We provide an example to analytically demonstrate the learning advantage of having the interference labels during training.
 Consider an environment of i.i.d. Bernoulli states $s_t = x_t$ with $P(x_t = 1) = P(x_t = 0) = 0.5$ and two actions $0$ and $1$.
 There is no reward taking action $a_t = 0$.
 When $a_t = 1$, the agent pays one unit to have
 a chance to win a two unit reward with probability $q_{x}$ at state $x_t  = x \in\{0, 1\}$.
 Therefore, $P(r_t = 1|x_t=x, a_t=1) = q_{x}$ and $P(r_t = - 1|x_t=x, a_t=1) = 1 - q_{x}$. This simple environment is a contextual bandit problem where the optimal policy is to pick $a_t = 1$ at state $x_t = x$ if $q_x > 0.5$, and $a_t = 0$ if $q_x \leq 0.5$.
 If the goal is to find an approximately optimal policy, the agent should take action $a_t = 1$ during training to learn the probabilities $q_0$ and $q_1$.
 Suppose the environment is subjected to observation black-out $x'_t = 0$ with $p^{\mathcal I} = 0.2$ when $x_t = 1$, and no interference when $x_t = 0$.
 Assume $q_0 = (3-q_1) / 5$. Then we have $P(r_t = 1 | x'_t = 1, a_t = 1) = q_1$, and 
 $P(r_t = 1 | x'_t = 0, a_t = 1) = 
 q_0P(x_t = 0| x'_t = 0) + q_1P(x_t = 1| x'_t = 0)
 = 0.5$.
If the agent only has the interfered observation $x'_t$, the samples for $x'_t = 0$ are irrelevant to learning $q_1$ because rewards just randomly occur with probability half given $x'_t = 0$.
Therefore, the sample complexity bound is proportional to $1 / P(x'_t = 1)$ because only samples with $x'_t = 1$ are relevant.
On the other hand, if the agent has access to the labels $i_t$ during training, even when observed $x'_t = 0$, the agent can infer whether $x_t = 1$ by checking $i_t = 1$ or not. Therefore, the causal relation allows the agent to learn $q_1$ by utilizing all samples with $x_t = 1$, and the \textbf{sample complexity bound is proportional} to $1 / P(x_t = 1) =  2$ which is a \textbf{20\% reduction} from $1 / P(x'_t = 1) =  2.5$ \textbf{when the labels are not available. }

Note that $z_t = (x_t, i_t)$ is a latent state for this example, and the latent state and its causal relation is very important to improving learning performance.

\subsection{Variational Causal Inference Q network (VCIQ)}
In addition, we provide an advanced discussion on using causal variational inference~\citep{louizos2017causal} for CIQ training, which is described as variational CIQ (VCIQ). As shown in Fig.~\ref{VCIQ}, VCIQ could be considered as a generative modeling based CIQ by using variational inference to model the latent information for Q-value estimation. Different from VAE-based Q-value estimation, VCIQ further incorporates the information of treatment estimation and outperforms its VAE-based ablations on Env$_2$ and Env$_3$ under the same model parameters. A basic implementation of VCIQ has been provided in our demonstration code for future studies as \textbf{the first} preliminary study of \textbf{a novel and effective variational architecture design on CIQ.}

\begin{figure}[ht!]
\centering
\includegraphics[width=0.80\textwidth]{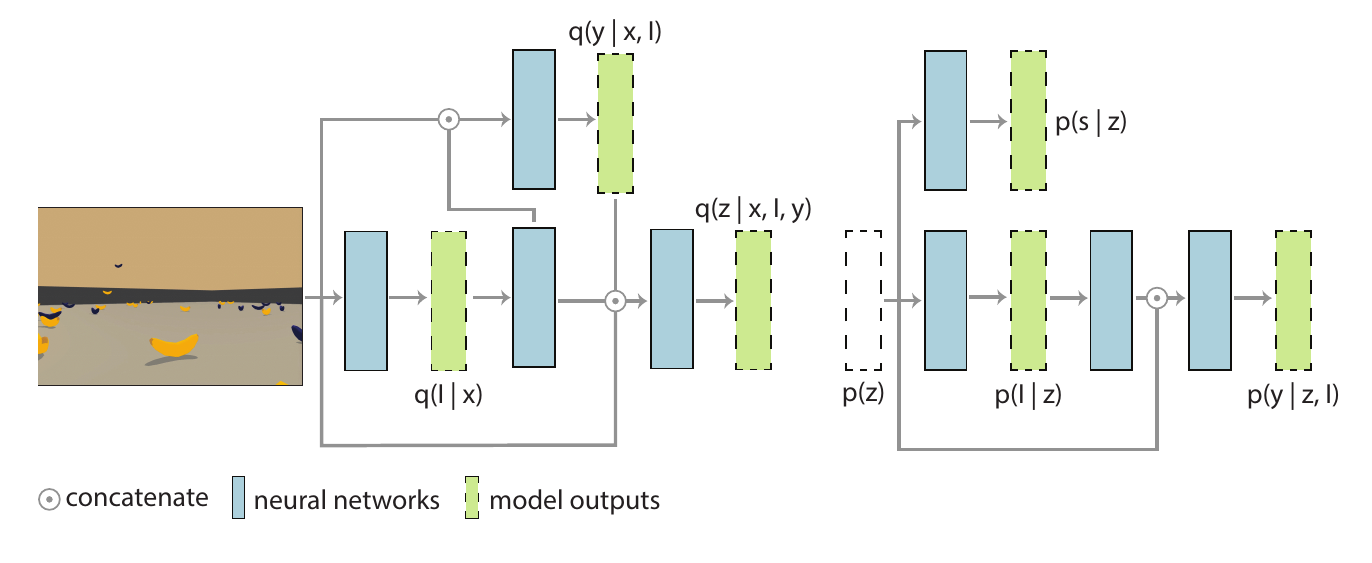}
\caption{VCIQ architecture. The notation $i^{train}_t$ denotes the inference label available during training, whereas $\tilde i_t$ is sampled during causal variational inference as $i_t$ is unknown.}
\label{VCIQ}
\end{figure}

\section{C. Causal Effects}
In a causal learning setting, evaluating treatment effects and conducting statistical refuting experiments are essential to support the underlying causal graphical model. Through resilient reinforcement learning framework, we could interpret DQN by estimating the average treatment effect (ATE) of each noisy and adversarial observation. 
We first define how to 
calculate a treatment effect in the resilient RL settings and conduct statistical refuting tests including random common cause variable test ($T_{c}$), replacing treatment with a random (placebo) variable ($T_{p}$), and removing a random subset of data ($T_{s}$). The open-source causal inference package Dowhy~\citep{sharma2019dowhy} is used for analysis.

\subsection{Average Treatment Effect under Intervention}
We refine a Q-network with discrete actions for estimating treatment effects based on Theorem 1 in~\citep{louizos2017causal}.
In particular, individual treatment effect (ITE) can be defined as the difference between the two potential outcomes of a Q-network; and the average treatment effect (ATE) is the expected value of the potential outcomes over the subjects. 
In a binary treatment setting, for a Q-value function $Q_t(s_t)$ and the interfered state $\mathcal I(s_t)$, the ITE and ATE are calculated by:
\begin{align}
    Q_{t}^{ITE}=Q_{t}(s_t)\left(1-p_{t}\right)+Q_{t}({\mathcal I}(s_{t})) p_{t} 
   \\ 
   A T E= \sum_{t=1}^{\mathcal T}\frac{\mathbb{E}\left[Q_{t}^{ITE}({\mathcal I}(s_t))\right]-\mathbb{E}\left[Q_{t}^{ITE}(s_t)\right]}{\mathcal T}
\end{align}
where $p_t$ is the estimated inference label by the agent and $\mathcal T$ is the total time steps of each episode. As expected, we find that CIQ indeed attains a better ATE and its significance can be informed by the refuting tests based on
$T_c$, $T_{p}$ and $T_s$. 

To evaluate the causal effect, we follow a standard refuting setting \citep{rothman2005causation, pearl2016causal,pearl1995testability} with the causal graphical model in Fig. 3 of the main context to run three major tests, as reported in Tab.~\ref{tab:ate:1}. The code for the statistical test was conducted by Dowhy~\citep{sharma2019dowhy}, which has been submitted as supplementary material. (We intend to open source as a reproducible result.)

Pearl \citep{pearl1995causal} introduces a "do-operator" to study this problem under intervention. The $do$ symbol removes the treatment $\mathbf{tr}$, which is equal to interference $\mathcal{I}$ in the Eq. (1) of the main content , from the given mechanism and sets it to a specific value
by some external intervention. The notation $P(r_{t}|do(tr))$ denotes the probability of reward $r_{t}$ with possible interventions on treatment at time $t$. Following Pearl’s back-door adjustment formula \citep{pearl2009causality} and the causal graphical model in Figure 2 of the main content., it is 
proved in \citep{louizos2017causal} that the causal effect for a given binary treatment $\mathbf{tr}$ (denoted as a binary interference label $i_{t}$ in Eq. (1) of the main content), a series of proxy variables $\mathbf{X}= (\sum^{\mathcal T}_{t=1} x_{t}) \equiv  \mathbf{S'}= (\sum^{\mathcal T}_{t=1} s'_{t}$), as $s'_{t}$ in Eq. (1) of the main content, a summation of accumulated reward $\mathbf{R} = (\sum^{\mathcal T}_{t=1} \mathbf{r_{t}})$ and a confounding variable $\mathbf{Z}$ can be evaluated by (similarly for $\mathbf{tr}=0$):
\begin{equation}
\begin{array}{l}{p(\mathbf{R} | \mathbf{S'}, d o(\mathbf{tr}=1))}= \int_{\mathbf{Z}} p(\mathbf{R} | \mathbf{S}, d o(\mathbf{tr}=1), \mathbf{Z}) p(\mathbf{Z} | \mathbf{S}, d o(\mathbf{tr}=1)) d \mathbf{Z}\stackrel{(i)}{=} \\
\int_{\mathbf{Z}} p(\mathbf{R} | \mathbf{S'}, \mathbf{tr}=1, \mathbf{Z}) p(\mathbf{Z} | \mathbf{S'}) d \mathbf{Z},\end{array}
\label{eq:sp:ite}
\end{equation}

where equality (i) is by the rules of do-calculus~\citep{pearl1995causal, pearl2016causal} applied to the causal graph applied on Figure 3 (a) of the main content. We extend to Eq.~\ref{eq:sp:ite} on individual outcome study with DQNs, which is  known by the Theorem 1. from Louizos et. al. ~\citep{louizos2017causal} and Chapter 3.2 of Pearl \citep{pearl2009causality}.

\subsection{Refutation Test:}
\label{sup:d:refute}
A sampling plan for collecting samples refer to as subgroups (i=1, ..., k). Common cause variation (T-c) is denoted as $\sigma_{c}$, which is an estimate of common cause variation within the subgroups in terms of the standard deviation of the within subgroup variation:
\begin{equation}
\sigma_{c} \cong \sum_{i=1}^{k} s_{i} / k ,
\end{equation}
where $k$ denotes as the number of sample size. We introduce intervention a error rate $n$, which is a probability to feed error interference (e.g., feed $i_{t}$ = 0 even under interference with a probability of $n$) and results shown in Table~\ref{tab:sp:ate}.

The test (T-p) of replacing treatment with a random (placebo) variable is conducted by modifying the graphical relationship in the proposed probabilistic model in Fig. 3 of the main context. The new assign variable will follow the placebo note but with a value sampling from a random Gaussian distribution. The test of removing a random subset of data (T-r) is to randomly split and sampling the subset value to calculate an average treatment value in the proposed graphical model.  
We use the official dowhy\footnote{Source:github.com/microsoft/dowhy/causal\_refuters} implementation, which includes: (1) confounders effect on treatment: how the simulated confounder affects the value of treatment; (2) confounders effect on outcome: how the simulated confounder affects the value of outcome; (3) effect strength on treatment: parameter for the strength of the effect of simulated confounder on treatment, and (4) effect strength on outcome: parameter for the strength of the effect of simulated confounder on outcome. Following the refutation experiment in the CEVAE paper, we conduct experiments shown in Tab.~S\ref{tab:sp:ate} and ~S\ref{tab:ate:1} with 10 $\%$ to 50 $\%$ intervention noise on the binary treatment labels. The results in Tab.~S\ref{tab:sp:ate} show that proposed CIQ maintains a lower rate compared with the benchmark methods included logistic regression and CEVAE (refer to Fig. 4 (b) in \citep{louizos2017causal}). 

Through Eq. (9) to (10) and the corresponding correct action rate in the main context, 
we could interpret deep q-network by estimating the average treatment effect (ATE) of each noisy and adversarial observation. ATE~\citep{louizos2017causal, shalit2017estimating} is defined as the expected value of the potential outcomes (e.g., disease) over the subjects (e.g., clinical features.) For example, in navigation environments, we could rank the harmfulness of each noisy observation~\citep{yang2019causal} against q-network from the autonomous driving agent. 

\begin{table}[ht!]
\centering
\caption{Absolute error ATE estimate; lower value indicates a much stable causal inference under perturbation on logic direction with $P^I = 10\%$ and $n$=error rate of intervention on the binary label.}
\label{tab:sp:ate}
\begin{tabular}{|l|lllll|}
\hline
\textbf{Model} & $n$=0.1 & $n$=0.2 & $n$=0.3 & $n$=0.4 & $n$=0.5\\ \hline
LR & 0.062 & 0.084 & 0.128 & 0.151 & 0.164 \\ \hline
CEVAE & 0.021 & 0.042 & 0.062 & 0.072 & 0.081 \\ \hline
CIQ & \textbf{0.019} & \textbf{0.020} & \textbf{0.015} & \textbf{0.018} & \textbf{0.023} \\ \hline
\end{tabular}
\end{table}

\begin{table}[ht!]
\caption{Validation of causal effect by three causal refuting tests. The causal effect estimate is tested by random common cause variable test (T-c),  replacing treatment with a random (placebo) variable (T-p -- lower is better), and removing a random subset of data (T-r). Adversarial attack outperforms in most tests.}
\label{tab:ate:1}
\centering
\adjustbox{max width=1\linewidth}{
\begin{tabular}{|l|llll|llll|}
\hline
Noise : $do(\mathcal I)$ & \multicolumn{4}{c|}{n = 0.1}                                                                 & \multicolumn{4}{c|}{n = 0.2}                                                                 \\ \hline
\textbf{Method}          & \multicolumn{1}{l|}{ATE} & \multicolumn{1}{l|}{w/ T-c} & \multicolumn{1}{l|}{w/ T-p} & w/ T-s & \multicolumn{1}{l|}{ATE} & \multicolumn{1}{l|}{w/ T-c} & \multicolumn{1}{l|}{w/ T-p} & w/ T-s \\ \hline
Adversarial              & 0.2432                   & 0.2431                      & 0.0294                      & 0.2488 & 0.0868                   & 0.0868                      & 0.0109                      & 0.0865 \\ \cline{1-1}
Black-out                & 0.2354                   & 0.2212                      & 0.0244                      & 0.2351 & 0.0873                   & 0.0870                      & 0.0140                      & 0.0781 \\ \cline{1-1}
Gaussian                 & 0.1792                   & 0.1763                      & 0.0120                      & 0.1751 & 0.0590                   & 0.0610                      & 0.0130                      & 0.0571 \\ \hline
Frozen Frame             & 0.1614                   & 0.1614                      & 0.0168                      & 0.1435 & 0.0868                   & 0.0868                      & 0.0140                      & 0.0573 \\ \hline
\end{tabular}}
\end{table}

\section{D. Ablation Studies}
\label{sup:e:ablation}
\subsection{The Number of Model Parameters}
We also spend efforts on a parameter-study on the results of average returns between different DQN-based models, which included DQN, Double DQN (DDQN), DDQN with dueling, CIQ, DQN with a classifier (DQN-CF), DDQN with a classifier (DDQN-CF), DQN with a variational autoencoder~\citep{kingma2013auto} (DQN-VAE), NoisyNet, and using the latent input of causal effect variational autoencoder for Q network (CEVAE-Q) prediction. Overall, CEVAE-Q is with minimal-requested parameters with 14.4 M (in Env$_{1}$) as the largest model used in our experiments in Tab.~\ref{tab:ab:over}. CIQ remains roughly similar parameters as 9.7M compared with DDQN, DDQN$_{d}$, and Noisy Net. Our ablation study in Tab.~\ref{tab:ab:over} indicates the advantages of CIQ are not owing to extensive features using in the model according to the size of parameters.  
CIQ attains benchmark results in our resilient reinforcement learning setting compared to the other DQN models. 

\begin{table}[ht!]
\centering
\caption{Ablation study on parameter of different DQN models using in our experiments in Env${_{1}}$, Env${_{2}}$, Env${_{3}}$, and Env${_{4}}$. }
\label{tab:ab:over}
\begin{tabular}{|l|l|llll|}
\hline
Model & Para. & Env$_{1}$ & Env$_{2}$ & Env$_{3}$ & Env$_{4}$ \\ \hline
DQN & 6.9M & 20.2 & 3.1 & -113.6 & 10.8 \\ \hline
DDQN & 9.7M & 41.1 & 3.5 & -123.4 & 57.9 \\ \hline
DDQN$_{d}$ & 9.7M & 82.9 & 4.7 & -136.3 & 67.2 \\ \hline
CIQ & 9.7M & \textbf{195.1} & \textbf{12.5} & \textbf{200.1} & \textbf{195.2} \\ \hline
DQN-CF & 9.7M & 140.5 & \textbf{12.5} & -78.3 & 120.2 \\ \hline
DDQN-CF & 12.1M & 161.3 & \textbf{12.5} & -10.1 & 128.2 \\ \hline
DQN-VAE & 9.7M & 151.1 & 7.6 & -92.9 & 24.1 \\ \hline
NoisyNet & 9.7M & 158.6 & 5.5 & 50.1 & 100.1 \\ \hline
CEVAE-Q & 12.5M & 39.8 & 11.5 & -156.5 & 45.8 \\ \hline
DVRLQ-CF & 10.7M & 107.11 & 9.2 & -34.9 & 42.5 \\ \hline
\end{tabular}
\end{table}

Noisy Nets \citep{fortunato2017noisy} has been introduced as a benchmark whose parameters are perturbed by a parametric noise function. We select Noisy Net in a DQN format as a noisy training baseline with interfered state $s'_t$ concated with a interference label $i_t$ from a classifier.

\subsection{Latent Representations}
\label{e2:latent}
We conduct an ablation study by comparing other latent representation methods to the proposed CIQ model.

\textbf{DQN with an variational autoencoder (DQN-VAE):}
To learn important features from observations, many recent works leverage deep variational inference for accessing latent states for feeding into DQN. We provide a baseline on training a variational autoencoder (VAE) built upon the DQN baseline, denoted as DQN-VAE. The DQN-VAE baseline is targeted to recover a potential noisy state and feed the bottleneck latent features into the Q-network. 

\textbf{CEVAE-Q Network:}
TARNet~\citep{shalit2017estimating, louizos2017causal} is a major class of neural network architectures for estimating outcomes of a binary treatment on linear data (e.g., clinical reports). Our proposed CIQ uses an end-to-end approach to learn the interventional (causal) features. We provide another baseline on using the latent features from a causal variational autoencoder~\citep{louizos2017causal} (CEVAE) as latent features as state inputs followed the loss function in~\citep{louizos2017causal}. To get the causal latent model in Q-network, we approximate the posterior distribution by a neural network $z_t \sim p(z_t|\tilde x_t)= \phi(\tilde x_t; \theta_1)$. Then we train this neural network, CEVAE-Q, by variational inference using the generative model.

We conduct 10,000 times experiments and fine-tuning on DQN-VAE and CEVAE-Q. 
The results in Table \ref{tab:sp:41} shows that the latent representation learned by CIQ provides better resilience than other representations.

\begin{table}[ht!]
\centering
\caption{Performance on average return in clean and five different noise level in Env$_{1}$ evaluated by an average of under uncertain perturbation included Gaussian, adversarial, blackout, and frozen frame. All DQN models solve the environment with over 195.0 average returns in a clean state input (a.k.a. no noise).}
\label{tab:sp:41}
\begin{tabular}{|l|llllll|c|c|}
\hline
Model  &0\% & 10\% & 20\% & 30\% & 40\% & 50\% & Cosine & Para.\\ \hline
DQN & 195.1 & 115.0 & 68.9 & 32.3 & 22.8 & 19.1 & 42.1 & 6.9 M \\ \hline
DDQN & 195.1 & 123.4 & 73.2 & 59.4 & 28.1 & 22.8 & 62.8 & 9.7 M \\ \hline
CIQ & 195.1 & \textbf{195.1} & \textbf{195.1} & \textbf{195.0} & \textbf{168.2} & \textbf{113.1} &   \textbf{195.0} & 9.7 M\\ \hline
DQN-VAE & 195.1 & 173.5 & 141.3 & 124.8 & 86.5 & 33.3 &101.2 & 9.7 M\\ \hline
DQN-CEVAE & 195.1 & 154.4 & 111.9 & 94.8 & 75.5 & 48.3 & 82.1 & 12.5 M\\ \hline
\end{tabular}
\end{table}

\begin{table}[ht!]
\centering
\caption{Structure-wise ablation studies of CIQ in Env$_1$ (noise level $P=20\%$). }
\label{tab:aba:2}
\begin{tabular}{|l|l|l|l|}
\hline
Model & Return & CLEVER-Q & AC-Rate \\ \hline
CIQ & 195.1 & 0.241 & 97.3 \\ \hline
B1: CIQ w/o the concatenation & 152.1 & 0.196 & 78.2 \\ \hline
B2: CIQ w/o the $\theta_3$ network & 150.1 & 0.182 & 65.6 \\ \hline
B3: CIQ w/o providing grounded $i_t$ for training & 135.1 & 0.142 & 53.6 \\ \hline
\end{tabular}
\end{table}

\subsection{Architecture Ablation Study on CIQ}
\label{e3:arch}
To study the importance of specific components in CIQ, we conducted additional ablation studies and constructed two new baseline models shown in Table~\ref{tab:aba:2} tested in Env$_1$ (Cartpole).
Baseline 1 (B1) - CIQ w/o the concatenation of $\tilde{i}_t$ in $S^C_I$. This comparison shows the importance of using both the predicted confounder $\tilde{z}_t$ and the predicted label $\tilde{i}_t$. B1 uses label prediction to help latent representation but not using the predicted labels in decision-making. The structure is motivated by a task-specific (depth-only information from a maze environment) DQN network from a previous study~\citep{mirowski2016learning}. 

Baseline 2 (B2) - CIQ w/o the $\theta_3$ network (for testing $\theta_3$’s importance).
 
 Baseline 3 (B3) - CIQ w/o providing grounded $i_t$ for training, for testing the importance of the inference loss and joint loss propagation. The superior performance of CIQ validates the proposed model is indeed crucial from the previous discussion in Section 3 of the main content.The setting used for Table 16 is the same as the setting for the third column (noise level = 20\%) in Table 5 and the third column (noise level = 20\%) in Table 15, tested in Env$_1$ (Cartpole).
 
We provide another variant of DQN-CF with ``two heads'' to account for the binary classifier output (perturbed/non-perturbed). We conduct experiments in Env$_1$ with Gaussian interference ($\mathcal{I}$=L$_2$) for this two-head DQN-CF variant, denoted as DQN-CF2H. 

As shown in Tab.~\ref{tab:rbt1}, DQN-CF2H achieves slightly better performance than DQN-CF at the higher noise levels P=$\{30\%, 40\%\}$, but slightly worse performance at the lower noise levels P=$\{10\%, 20\%\}$. CIQ still performs the best in terms of average rewards and CLEVER-Q scores, suggesting the benefits of our method beyond the architectural difference.

\begin{table}[ht!]
\centering
\caption{Evaluation of $\{$average rewards$/$CLEVER-Q$\}$ in Env$_1$.}
\label{tab:rbt1}
\begin{tabular}{|c|c|c|c|c|}
\hline
Model, $\mathcal{I}$=L$_2$  & P=10\% & P=20\% & P=30\% & P=40\% \\ \hline
DQN-CF & 192.8$/$0.185 & 147.7$/$0.145 & 131.4$/$0.127 & 88.2$/$0.92 \\ \hline
DQN-CF2H & 188.4$/$0.180 & 145.2$/$0.141 & 135.2$/$0.131 & 92.1$/$0.97 \\ \hline
CIQ & \textbf{195.1}$/$\textbf{0.221} & \textbf{195.1}$/$\textbf{0.235} & \textbf{195.0}$/$\textbf{0.232} & \textbf{168.2}$/$\textbf{0.186} \\ \hline
\end{tabular}
\end{table}

\subsection{Perturbation-based Neural Saliency for DQN Agents}
\label{sup:e:saliency:map}
To better understand our CIQ model, we use the benchmark saliency method on DQN agent, perturbation-based saliency map, \citep{greydanus2018visualizing} to visualize the salient pixels, which are sensitive to the loss function of the trained DQNs. We made a case study of an input frame under an adversarial perturbation, as shown in Fig. \ref{fig:figure:4}. We evaluate DQN agents included DQN, CIQ, DQN-CF, DVRLQ-CF and record its weighted center from the neural saliency map, where saliency pixels of CIQ respond to ground true actions more frequent (96.2\%) than the other DQN methods. \\

\textbf{Predicted interference label accuracy}:  We use the switching mechanism for end-to-end DQN training. The major difference between CIQ and enhanced DQNs is how to use the interference information for learning, which results in different architectures. We provide a related case study in Table~\ref{tab:acc}. Considerall the task and noisy conditions, the Pearson correlation coefficient between prediction accuracy and average returns is 0.1081, which shows low statistical correlation.\\

\textbf{Baselines discussion considering the variance of average rewards}: We follow the standard DQNs performance evaluation in the DRL community by using average returns to evaluate the performance for DQNs. CIQ’s average returns outperform the other DQNs. Moreover, CIQ shows advantage on evaluated robustness properties including CLEVER-Q and action correction rate. We calculate the p-value of the learning curves of CIQ between the other DQNs over all evaluated noisy conditions and environments. In general, the p-value <0.05 could be considered as statistically significant in Table~\ref{tab:pvalue}. \\

\textbf{White-out Ablation:} We tested the performance of CIQ against unseen (not used int raining) interferences and showed that CIQ has improved robustness. In our white experiments, from scratch training with white-out perturbation, CIQ remains the best in all deployed environments. The average reward for white-out (e.g., to give each state observation a max-available intensity) agents showed a relative 12.3\% decay from black-out perturbation (e.g., set each state observation as zero). 

\begin{table}[ht!]
\centering
\caption{Additional ablation study with noisy environments (P= 20\%) for (continuous control Cartpole).The number shows prediction accuracy on the interference label (\%). }
\label{tab:acc}
\begin{tabular}{|l|l|l|l|l|}
\hline
\multicolumn{1}{|c|}{\textbf{Model}} & \multicolumn{1}{c|}{\textbf{Gaussian}} & \multicolumn{1}{c|}{\textbf{Adversarial}} & \multicolumn{1}{c|}{\textbf{Blackout}} & \multicolumn{1}{c|}{\textbf{Frozen}} \\ \hline
DQN-CF & 98.34 & 95.34 & 100.00 & 91.23 \\ \hline
DVRLQ-CF & 97.34 & 95.52 & 100.00 & 90.92 \\ \hline
CIQ & 98.21 & 95.67 & 100.00 & 90.82 \\ \hline
\end{tabular}
\end{table}

\begin{table}[ht!]
\centering
\caption{p-value of the learning curves collected different environments and noisy types and levels presented in this paper. The results suggested that the learning curves of CIQs (best performance in terms of the average returns, CLEVER-Q, action correction rate) could be considered as statistically significant. }
\label{tab:pvalue}
\begin{tabular}{|l|l|l|l|l|l|}
\hline
\multicolumn{1}{|c|}{\textbf{Model}} & \multicolumn{1}{c|}{\textbf{DQN}} & \multicolumn{1}{c|}{\textbf{DQN-CF}} & \multicolumn{1}{c|}{\textbf{DQN-SA}} & \multicolumn{1}{c|}{\textbf{DVRLQ}} & \multicolumn{1}{c|}{\textbf{DVRLQ-CF}} \\ \hline
CIQ & 0.0112 & 0.0291 & 0.0212 & 0.0142 & 0.0183 \\ \hline
\end{tabular}
\end{table}

\subsection{Robustness Transferability among Different Interference Types}
\label{sup:sec:trans:infer}
We conduct additional experiments to study robustness transferability of DQN and CIQ when training and testing under different kinds of interference types in Env$_1$. Note that both architectures would solve a clean environment successfully (over 195.0). The reported numbers are averaged over 20 independent runs for each condition. As shown in Table~\ref{tab:transfer:DQN} and Table~\ref{tab:transfer:CIQ}, CIQ agents consistently attain significant performance improvement when compared with DQN agents, especially between Gaussian and adversarial perturbation.
For example, CIQ succeeded to solve the environment 12 times out of 20 independent runs, with an average score of 165.2 in Gaussian (train)-Adversarial (test) adaptation.  
n particular, for CIQ, 12 times out of 20 independent runs are successfully transfered from Gaussian to Adversarial perturbation. Interestingly, augmenting adversarial perturbation does not always guarantee the best policy transfer when testing in the Blackout and Frozen conditions, which shows a slightly lower performance compared with training on Gaussian interference. The reason could be attributed to the recent findings that adversarial training can undermine model generalization~\citep{raghunathan2019adversarial,su2018robustness}.

\begin{table}[ht!]
\centering
\caption{DQN adaptation: train and test on different interference (noise level $P = 20\%$) in Env$_1$.}
\label{tab:transfer:DQN}
\begin{tabular}{|l|l|l|l|l|}
\hline
Train \textbf{/ Test} & \textbf{Gaussian} & \textbf{Adversarial} & \textbf{Blackout} & \textbf{Frozen} \\ \hline
Gaussian & 67.4 & 38.4 & 43.7 & 52.1 \\ \hline
Adversarial & 53.2 & 42.5 & 35.3 & 44.2 \\ \hline
Blackout & 46.2 & 27.4 & 85.7 & 50.3 \\ \hline
Frozen & 62.3 & 26.2 & 45.9 & 62.1 \\ \hline
\end{tabular}
\end{table}

\begin{table}[ht!]
\centering
\caption{CIQ adaptation: train and test on different interference (noise level $P = 20\%$) in Env$_1$. }
\label{tab:transfer:CIQ}
\begin{tabular}{|l|l|l|l|l|}
\hline
Train \textbf{/ Test} & \textbf{Gaussian} & \textbf{Adversarial} & \textbf{Blackout} & \textbf{Frozen} \\ \hline
Gaussian & \textbf{195.1} & 165.2 & 158.2 & 167.8 \\ \hline
Adversarial & 162.8 & \textbf{195.0} & 152.4 & 162.5 \\ \hline
Blackout & 131.3 & 121.1 & \textbf{195.3} & 145.7 \\ \hline
Frozen & 161.6 & 135.8 & 147.1 & \textbf{195.2} \\ \hline
\end{tabular}
\end{table}

\subsection{CIQ with Multi-Interference.}
\label{sup:sec:e6:multi:infer}

Here we show how the proposed CIQ model can be extended from the architecture shown in Figure~\ref{Nets} to the multi-interference (MI) setting.
The design intuition is based on two-step inference by a common encoder, to infer a clean or noisy observation, followed by an individual decoder tied to an interference type, to infer noisy types and activate the corresponding Q-network (named $\theta_4$).
\\
Note that the two-step inference mechanism follows the RCM as two sequential potential outcome estimation models~\citep{rubin1974estimating, imbens2010rubin}, where interfered observation $x'_t$ is determined by two labels $i_{1,t}$ and $i_{2,t}$ according to $x'_t = i_{1,t} (i_{2, t} \mathcal I_1(x_t) + (1 - i_{2, t})\mathcal I_2(x_t) )  + (1 - i_{1, t}) x_t$ extended from Eq.(\ref{eq:noise:41}), where $i_{1,t}$ indicates the presence of interference and $i_{2,t}$ indicates which interference type (here we show the case of two types).
As a proof of concept, we consider two interference types together, Gaussian noise and adversarial perturbation. In this setting every observation (state) can possibly undergo an interference with either Gaussian noise or Adversarial perturbation.
From the results shown in Table.~\ref{tab:CIQ:MI}, we find that the extended version of CIQ, CIQ-MI, is capable of making correct action to solve (over 195.0) the environment when training with mixed interference types (last row). Another finding is that  robustness tranferability (153.9/154.2) in CIQ-MI is slightly degraded compared to the results (162.8/165.2) in Table~\ref{tab:transfer:CIQ} with the same training episodes ($500$) and runs ($20$), which could be caused by the increased requirement of model capacity~\citep{ammanabrolu2019transfer, yang2020multi} of CIQ-MI.

\subsection{Hardware Setup and Energy Cost}
We use Nvidia GPUs (2080-Ti and V100) for our experiments with Compute Unified Device Architecture (CUDA) version 10.1. To conduct the results shown in the paper, it takes around 20 min to run 1,000 epochs (maximum) with a batch size 32 for each environment considering the hyper-parameters tuning  described in Section \ref{sup:c:traing:setting} of the main paper. In total, the experiments presented (four environment with four different types of noise and its ablation studies) in this paper took around 
343 wall-clock hours with a 300W power supplier.
\clearpage
\section*{Broader Impact}
 With the recent advances in using (deep) RL to solve problems that were once believed to be challenging for machines to learn and understand, we believe it is timely to move to the next milestone: understanding the resilience of DRL to rare but possible and recurring noisy interferences, which motivates this work with a novel deign of Q-networks inspired from causal learning.
 \begin{itemize}
     \item \emph{Who may benefit from this research:} Researchers  working on RL technology; as well as the users using the related technology for responsible and safe machine learning technology.
     \item \emph{Who may be put at disadvantage from this research: :} When the work discloses the findings that Q-networks can be sensitive to noisy interferences, we understand the responsibilities of explaining the results to the public properly and providing reproducible evaluation. 
     \item \emph{Whether the task/method leverages biases in the data:} To alleviate possible bias in the data and model, we put efforts toward designing reproducible metrics and evaluating a wide variety of reproducible environments, interference types, and several baseline models. 
 \end{itemize}

\end{document}